\documentclass[final,12pt]{colt2023} % Anonymized submission
\usepackage{booktabs} % For formal tables
\usepackage{amsmath}
\usepackage{mathtools}
\usepackage{amsfonts}
\usepackage{bm}
\usepackage{algorithm}
\usepackage{algorithmic}

\DeclareMathOperator*{\argmax}{arg\,max}
\DeclareMathOperator*{\argmin}{arg\,min}
\DeclareMathOperator{\var}{Var}
\newcommand{\PR}{\mathbb P}
\newcommand{\EX}{\mathbb E}
\newcommand{\Ind}{\mathbb{I} }
\newcommand{\tA}{A_{ti}}
\newcommand{\sA}{A_{si}}
\newcommand{\tY}{Y_{ti}}
\newcommand{\sY}{Y_{si}}
\newcommand{\tS}{S_{ti}}

\newcommand{\tRtwo}{R_{ti}^{(2)}}

\newcommand{\tRone}{R_{ti}^{(1)}}

\newcommand{\PP}{\mathcal P}
\newcommand{\muone}{\hat{\mu}_{ti}^{(1)}}
\newcommand{\mutwo}{\hat{\mu}_{ti}^{(2)}}
\newcommand{\muthree}{\hat{\mu}_{ti}^{(3)}}
\newcommand{\sumst}{\sum_{s=1}^t}
\newcommand{\sumiK}{\sum_{i=1}^K}
\newcommand{\muopt}{\mu_{i^*}}

\newcommand{\RV}[1]{#1}
\setcitestyle{authoryear}

\title[Resources Allocation]{Allocating Divisible Resources on Arms with Unknown and Random Rewards}
\usepackage{times}
\coltauthor{%
	\Name{Ningyuan Chen} \Email{ningyuan.chen@utoronto.ca}\\
	\addr Department of Management, University of Toronto Mississauga, ON L5L 1C6, Canada\\
	Rotman School of Management, University of Toronto, ON M5S 3E6, Canada
	\AND
	\Name{Wenhao Li} \Email{zjuliwenhao@gmail.com}\\
	\addr College of Business, Shanghai University of Finance and Economics, Shanghai 200433, China
}
\begin{document}
	
	\maketitle
	\begin{abstract}
		We consider a decision maker allocating one unit of renewable and divisible resource in each period on a number of arms. The arms have unknown and random rewards whose means are proportional to the allocated resource and whose variances are proportional to an order $b$ of the allocated resource.
		In particular, if the decision maker allocates resource $A_i$ to arm $i$ in a period, then the reward $Y_i$ is
		$Y_i(A_i)=A_i \mu_i+A_i^b \xi_{i}$,
		where $\mu_i$ is the unknown mean and the noise $\xi_{i}$ is independent and sub-Gaussian.
		When the order $b$ ranges from 0 to 1, the framework smoothly bridges the standard stochastic multi-armed bandit and online learning with full feedback. We design two algorithms that attain the optimal gap-dependent and gap-independent regret bounds for $b\in [0,1]$, and demonstrate a phase transition at $b=1/2$. The theoretical results hinge on a novel concentration inequality we have developed that bounds a linear combination of sub-Gaussian random variables whose weights are fractional, adapted to the filtration, and monotonic.\footnote{Extended abstract accepted for presentation at the Conference on Learning Theory (COLT) 2023}
	\end{abstract}
	\begin{keywords}%
		renewable and divisible resource allocation, stochastic multi-armed bandit, gap-dependent (independent) regret
	\end{keywords}

	% Paper body
	\section{Introduction}
		Consider the following three real-world examples of online decision-making problems.	
	\begin{example}[Time Management]\label{exa:TM}
		A PhD student needs to decide how to allocate her time in a day. 
		She can spend five hours exploring a few research directions and the rest on coursework.
		But the payoff of studying the subjects is unknown and depends on the time invested.
		Fortunately, when the sun rises again, she can plan her second day based on the feedback she has obtained from the first day, and the resource (a day's time in this case) is renewed.
		The objective is to find the best project to spend time on in the long run and maximize the reward in terms of intellectual fulfillment.  
	\end{example}
	
	\begin{example}[Venture Capital Investment]\label{exa:VC}
		A venture investor raises a fixed amount of funds and invests it into multiple startup companies every month. 
		Initially, the average returns of the companies are unknown.
		The performance and thus the monthly return of the projects depend on the capital invested.
		The objective is to learn the value of the companies and divert the fund to the project(s) with the highest returns in the long run.
	\end{example}
	
	\begin{example}[Vaccination Allocation]\label{exa:VA}
		A government agency allocates COVID test kits to different cities to control disease transmission. 
		The total number of testing kits available every day is fixed and
		the ``reward'' is the number of positive cases identified by the tests.
		The agency needs to learn the prevalence of each city by observing the test results
		and allocate more test kits to the cities with higher infection rates.
	\end{example}
	At a high level, all three examples share a few common features.
	The decision-making process is dynamic;
	in each period, there is a fixed amount of \emph{renewable} resource (time, money, test kits) that can be \emph{divided} and allocated to multiple arms (projects, companies, cities);
	the arms generate random rewards whose means are unknown initially and depend on the allocated resource. 
	
	Except for the divisible resources, the problem is similar to the stochastic multi-armed bandit (SMAB) problem studied extensively in the literature.
	In SMAB, in each period, the decision maker has one unit of renewable resource, allocates the resource to exactly one of the $K$ arms, and observes the reward. 
	However, in the above three examples, the divisible resources are a key feature that cannot be ignored.
	This feature introduces fundamental and complex changes to SMAB: 
	the decision space becomes uncountably infinite in each period;
	how the reward of an arm depends on the allocated resource also becomes an important characteristic of the problem.
	In this paper, we tackle this decision-making problem by proposing a framework to incorporate divisible resources into SMAB.
	In particular, if the decision maker allocates resource $\tA \in [0,1]$ to arm $i \in [K]$ in period $t$, then the reward $\tY$ is
	$\tY(\tA)=\tA \mu_i+\tA^b \xi_{ti}$,
	where $\mu_i$ is the unknown mean and the noise $\xi_{ti}$ is independent and sub-Gaussian with parameter $\sigma$. 
	Thus, the expected reward is proportional to the allocated resource, while the scale of the noise is proportional to an order $b$ of the resource.	
	As we shall see, $b$ reflects the signal-to-noise ratio (SNR) of the reward.
	
	% we don't use itemize to save some space for colt
	Developing theories upon the framework, this paper makes the following contributions to the literature.
	First, we develop two algorithms for the problem, inspired by the design principles of successive elimination and $\epsilon$-greedy algorithms, for the gap-independent and gap-dependent regret, respectively.
	We show that the algorithms attain the optimal rate of gap-independent and gap-dependent regret for $b\in(0,1)$. (See Table~\ref{tab:regret-summary} for the regret rates.)
	The regret leads to a number of interesting findings.
	(1) the regret displays completely different behavior for $b\le 1/2$ and $b>1/2$ and thus phase transition at $b=1/2$.
	(2) the gap-dependent regret is $O(\log T)$ for $b\le 1/2$ and \emph{finite} for $b>1/2$. 
	For the gap-independent bound, a larger $b>1/2$ reduces the regret in terms of the order of $K$ but not $T$.
	(3) the regret smoothly bridges that of SMAB for small SNR ($0\le b\le 1/2$) and that of online learning with full feedback for large SNR ($b=1$).
	
	Second, in the theoretical analysis, we establish a novel concentration inequality that bounds a linear combination of sub-Gaussian random variables whose weights are fractional and monotonic, which has not been discovered in the literature. 
	In particular, consider independent $\sigma$-sub-Gaussian random variables $\{\xi_s\}_{s=1}^t$ and a sequence of deterministic weights $\{a_s\}_{s=1}^t$ such that $0\le a_1\le a_2\le \ldots\le a_t$.
	We prove that for any constant $\epsilon>0$, we have
	\begin{equation}\label{eq:mono-conc-ineq}
		\PR\left( \sup_{0\le a_1\le \ldots\le a_t}\frac{\sum_{s=1}^t a_{s} \xi_{s}}{\sqrt{\sum_{s=1}^t a_{s}^2}} \ge  \sqrt{\frac{3}{2}}   (\log t) \epsilon \right) \le \exp\left(-\frac{\epsilon^2}{2 \sigma^2}\right).
	\end{equation}
	The supremum allows the concentration inequality to be adapted to random weights.
	To draw connections to standard concentration inequalities, we note that for a given sequence $\{a_s\}_{s=1}^t$, one has 
	\begin{equation}\label{eq:standard-conc-ineq}
		\PR\left( \frac{\sum_{s=1}^t a_{s} \xi_{s}}{\sqrt{\sum_{s=1}^t a_{s}^2}} \ge \epsilon\right) \le \exp\left(-\frac{\epsilon^2}{2 \sigma^2}\right).
	\end{equation}
	Thus, taking the supremum over the set $\{0\le a_1\le \dots\le a_t\}$ doesn't inflate the concentration probability significantly,
	because both are concentrated within $O((\log t)^{3/2})$.
	On the other hand, if we remove the monotonicity constraint, the random quantity is significantly less concentrated: 
	\begin{equation}\label{eq:nonmono-conc-ineq}
		\EX\left( \sup_{a_1, \ldots,  a_t\ge 0}\frac{\sum_{s=1}^t a_{s} \xi_{s}}{\sqrt{\sum_{s=1}^t a_{s}^2}}\right) =O\left(\|\bm{\xi}\|_2\right)=O(\sqrt{t})
	\end{equation}
	Therefore, monotonicity is the key to the concentration probability.
	To prove the inequality, we use semidefinite programming and other techniques. See Proposition~\ref{prop:concen-mono} for details. 
	To summarize the challenges in the analysis, note that the union bound combined with \eqref{eq:standard-conc-ineq} is not sufficient for \eqref{eq:mono-conc-ineq}, because the supremum is taken over an infinite set.
	A common workaround in nonparametric statistics is to use the covering number and apply the union bound to the centers of the covering balls. 
	However, the covering number of the monotonic set $\{0\le a_1\le \dots\le a_t\le 1\}$ (properly normalized) is not qualitatively smaller than that of the hypercube $[0,1]^t$ to give a much sharper concentration inequality (\eqref{eq:nonmono-conc-ineq} versus \eqref{eq:mono-conc-ineq}).
	The concentration inequality \eqref{eq:mono-conc-ineq} and its analysis may be of independent interest.
	% \item Under the framework, we consider the gap-independent regret which only depends on time horizon $T$ and the number of projects $K$, and gap-dependent regret which further depends on the suboptimality gap of arm $i$, i.e., $\Delta_i=\max_j{\mu_j}-\mu_i$.  We prove the lower bounds for both types of regret. We  propose new algorithms attaining the tight regret bounds. The regret bounds are summarized in Table \ref{tab:regret-summary}.
	% \item In the spectrum of the constant $b \in [0,1]$, we find a phase transition at $b=1/2$. When $b \in \left(0,1/2\right]$, the optimal regret of our problem is the same as SMAB. It implies we can apply the standard algorithms such as Upper Confidence Bound (UCB) to achieve the optimal regret. When $b \in \left(1/2,1\right)$, we show the optimal regret of our problem smoothly interpolates the regret of SMAB and full-feedback case. 
	\begin{table}[t]
		\caption{Summary of regret bounds}
		\label{tab:regret-summary}
		\centering
		\begin{tabular}[c]{lcc}
			\toprule
			& Gap-independent & Gap-dependent \\
			\midrule
			SMAB ($b=0$) \citep{auer2002finite} & $O\left(\sqrt{TK}\right)$ & $O\left(\log T \sum_{i}\Delta_i^{-1}\right)$ \\
			$b \in \left(0,1/2\right]$ (this work)  & $O\left(\sqrt{TK}\right)$ & $O\left(\log T \sum_{i} \Delta_i^{-1}\right)$ \\
			$b \in (1/2,1)$ (this work) &  $O\left(\sqrt{T} K^{1-b}\right)$ & $O(1)$ \\
			Full feedback ($b=1$) \citep{degenne2016anytime} & $O\left(\sqrt{T\log K}\right)$ &  $O(1)$\\ %$O\left(\sum_{i}\Delta_i^{-1}\right)$ \\
			\bottomrule
		\end{tabular}
	\end{table}
	
	\subsection{Related Work}
	Our paper is closely related to the literature on SMAB \citep{auer2002finite,bubeck2012regret,lattimore2020bandit,slivkins2019introduction,agrawal21a} and multi-armed bandit with full-information feedback \citep{degenne2016anytime,slivkins2019introduction,zhang2019online,huang22a},
	as our framework smoothly bridges the two settings.
	There are works studying the middle ground of the two settings:  
	\citet{degenne2018bandits} assumes the decision maker has extra and free observations from time to time; \citet{bubeck2013prior,locatelli2016optimal,yang2020optimal} suppose the decision maker has prior knowledge on the optimal mean reward or the suboptimality gaps. 
	\RV{\citet{wu2015online} study a semi-bandit setting where the learner observes the rewards of other arms with the same mean.}
	However, these works consider indivisible resources like SMAB and special cases between the two settings. 
	% In our framework, the divisiblity of the resource is the key point which helps to unite the setting of bandit and full feedback.
	%As far as know, our paper is among the first to propose and analyze a framework that unifies the setting of bandit and full feedback.
	The literature on MAB with knapsacks consider the resource constraints in online learning problem \citep{agrawal2014bandits,agrawal2016linear, badanidiyuru2018bandits,merlis20a,kesselheim20a,sivakumar2022smoothed}. 
	Different from this work, the resource constraint is imposed over the time horizon. 
	% and the resource is not replenishable. The learning will stop when the resource is exhausted. So their settings are quite different from ours.
	
	The literature on exponential bandits \citep{keller2005strategic, keller2010strategic, chen2020multi,marlats2021strategic} studies the allocation of renewable and divisible resources to multiple arms. %\lwh{Lwh:Since we change the title to ``arms'', should we change the ``projects'' here and in formulation section to ``arms''?}
	The setting, however, is quite different: the decision maker terminates the process as soon as an arm generates a positive reward.
	%right after a positive reward is generated by an arm. 
	A Bayesian framework is commonly adopted to characterize the optimal policy.
	Our paper differs from this literature in the problem setting and the performance measure. \RV{\citet{mandelbaum1987continuous} study a continuous-time setting of our problem where all the arms share the divisible resources. However, their objective is to provide a policy that maximizes the discounted reward, and they demonstrate that the optimal policy adopts a Gittins-Index structure. In contrast, our paper focuses on a different reward structure, and our objective is to develop algorithms to achieve the optimal rate of regret. }
	% In our framework, the decision maker keeps playing the game until the last period and aims to minimize the total regret.
	
	Our work is mostly related to the literature on online resource allocation with semi-bandit feedback \citep{lattimore2014optimal,lattimore2015linear,dagan2018better,verma2019censored,fontaine2020adaptive,sherman21a}. 
	Like our work, this literature focuses on the online allocation of renewable and divisible resources to multiple arms with a resource constraint in each period. 
	However, the reward models in these papers are quite different from ours. 
	\citet{lattimore2014optimal,lattimore2015linear,dagan2018better} assume that the reward $\tY(\tA)$ follows a Bernoulli distribution with a mean of $\min\{1,\tA/ \mu_i\}$, motivated by the problem of allocating the computing resources (cache, bandwidth, CPU, etc) to multiple processes. 
	They show the optimal gap-dependent and gap-independent regret grow at a rate of $O(\log T/ \Delta_i)$ and $O(\sqrt{KT})$, which matches SMAB. 
	\citet{verma2019censored} further consider an additional threshold parameter $\theta_i$ in the reward model. 
	\citet{fontaine2020adaptive,sherman21a} extend the reward model to general concave functions. 
	Our work differs from these papers in terms of the reward model, algorithm design and regret rate. 
	We adopt a constant $b$ to reflect the SNR, which allows the framework to smoothly bridges SMAB and full feedback. 
	The value of $b$ dictates the algorithmic choice and optimal regret, and the phase transition at $b=1/2$ is a unique feature of our model.
	% Specifically, when $b$ is large enough ($b>1/2$), we can take advantage of the divisible resource to design algorithms to achieve much smaller regret than SMAB.

	\section{Problem Formulation}\label{sec:formu}
	We consider a resource allocation problem with $K$ arms and a time horizon of length $T$. At each period $t\in \left\{1,2,\ldots,T\right\}$, the decision maker allocates a resource $\tA\in[0,1]$ to arm $i\in[K]=\{1,2,\ldots,K\}$ subject to the constraint that the total amount of the resource is one, i.e., $\sum_{i=1}^K \tA = 1$. For each arm $i$, a reward $\tY$ is observed and collected.
	%We define the set of arms as $[K] \coloneqq \{1,2,\ldots,K\}$ and the total length of the time horizon as $T$. 
	%In each period $t \in [T] \coloneqq \{1,2,\ldots,T\}$, the decision maker allocates the resource $\tA \in [0,1]$ for arm $i \in [K]$ with the constraint that the total amount of the resource in this period is one, i.e., $\sum_{i=1}^K \tA=1$. 
	%For each arm $i$, a reward $\tY$ is observed and collected. 
	The mean of $\tY$ is proportional to the allocated resource $\tA$ and the variance is proportional to an order $b$ of the resource, given by the following equation, 
	\begin{equation}\label{equ:def-feedback}
		\tY(\tA)=\tA \mu_i+\tA^b \xi_{ti},
	\end{equation}
	where the constant $\mu_i \in \left[0,\infty\right)$ is the mean reward of arm $i$ and the noises $\{\{\xi_{ti}\}_{t=1}^T\}_{i=1}^K$ are assumed to be independent and sub-Gaussian with parameter $\sigma$. The use of sub-Gaussian noises are standard assumptions in online learning problems. 
	
	\textbf{Policy.} The decision maker does not know $\{\mu_i\}_{i=1}^K$ initially. At each period $t \in [T]$, she uses the past history 
	\begin{equation}\label{equ:his-def}
		H_t=\left\{A_{11},A_{12},\ldots,A_{1K},Y_{11},\ldots,Y_{1K},\ldots,A_{t-1,1},\ldots,A_{t-1,K},Y_{t-1,1},\ldots,Y_{t-1,K}\right\},
	\end{equation} 
	to implement a the policy $\pi_t$, which is a mapping from $ \mathbb{R}^{2(t-1)K}$ to $[0,1]^K$. The policy determines the allocation of resources $\tA=(\pi_t(H_t))_i$ to arm $i$ at period $t$.
	%so that $\tA=(\pi_t(H_t))_i$ is allocated to arm $i$ at period $t$. 
	%	Given the mean $\mu_i$, noises $\xi_{ti}$, policy $\pi_t$, the online learning problem in \ref{equ:def-feedback} is well-defined.
	
	\textbf{Regret.} The goal is to maximize the expected total reward collected over all periods: 
	\begin{equation*}
		\EX\left[\sum_{t=1}^T \sum_{i=1}^K \tY(\tA)\right]= \EX\left[\EX\left[\sum_{t=1}^T \sum_{i=1}^K \tY(\tA) \bigg| H_t\right]\right]= \EX\left[\sum_{t=1}^T \sum_{i=1}^K \tA \mu_i\right].
	\end{equation*}
	As in SMAB, we benchmark the expected total reward against the total reward of the best arm $T \max_{i \in [K]} \mu_i$ if $\{\mu_i\}_{i=1}^K$ were known. Therefore, we define the best arm $i^*=\argmax_{i \in [K]} \mu_i$ and the regret of a policy $\pi$  
	\begin{equation}\label{equ:def-regret}
		R_{\pi}(T)=\sum_{t=1}^T \EX\left[ \mu_{i^*}-\sum_{i=1}^K \tA \mu_i\right]=T \mu_{i^*}-\EX\left[\sum_{t=1}^T \sumiK \tA \mu_i\right],
		% =\sum_{t=1}^T \sumiK \tA \Delta_i,
	\end{equation}
	The objective of the decision maker is thus designing a policy that achieves small regret.
	In the literature, there are two types of regret bounds for gap-independent and gap-dependent cases. 
	We provide a summary here.
	In the gap-independent case, we focus on the minimization of \eqref{equ:def-regret} for an arbitrary problem instance, including $(\mu_1,\dots,\mu_K)$ and the noise distribution.
	As a result, the regret depends only on $T$ and $K$.
	In the gap-dependent case, the mean reward $(\mu_1,\dots,\mu_K)$ is fixed but unknown to the decision maker. The regret typically depends on the suboptimality gap $\Delta_i\coloneqq \mu_{i^*}-\mu_i$, in addition to $T$ and $K$.

	\RV{\textbf{Motivation for the Reward Structure}. To motivate the reward structure in \eqref{equ:def-feedback}, we can think of period $t$ in our model as a batch of many bandit decision epochs. The proportion of decision epochs dedicated to arm $i$ is $A_{ti}$. If the random reward in each decision epoch is independent, identically distributed, and has a finite variance (normalized by the epoch length), then by the central limit theorem (CLT), the total reward from arm $i$ follows \eqref{equ:def-feedback} with $b=1/2$  when we take the epoch length to zero. To see this, let's suppose there are $n$ epochs in period $t$. If we allocate all the $n$ epochs to arm $i$, we have $A_{ti}=1$ and $\EX[\tY]=\mu_i$, $\var(\tY) = \sigma^2$. 
	Then, we consider the random reward in each epoch, denoted by $\{X_j\}_{j=1}^n$. Since $X_j$ are i.i.d. and $\sum_{j=1}^n X_j=\tY$, 
	%The reward $\tY$ is evenly divided into $n$ epochs, as is its mean and variance. Thus, the random reward in epoch $j$, denoted by $X_j$, satifies
	we have $\EX[X_j]=\EX[\tY]/n=\mu_i/n$, $\var(X_j) =\var(\tY)/n = \sigma^2/n$. If allocating $n_i$ epochs to arm $i$, we have $A_{ti}=n_i/n \le 1$ and $Y_{ti}=\sum_{j=1}^{n_i} X_j$. Let $n \rightarrow \infty$ while maintaining $A_{ti}$ as a constant. Then, by the CLT, we have $Y_{ti}=\sum_{j=1}^{n_i} X_j \overset{d}{\rightarrow} \mu_i n_i/n + \sqrt{n_i} N(0, \sigma^2 /n)=A_{ti} \mu_i+ \sqrt{A_{ti}} N(0,\sigma^2)$, which follows \eqref{equ:def-feedback} with $b=1/2$.}
	
	\RV{Therefore, when $b=1/2$ and $n \rightarrow \infty$, the reward $Y_{ti}$ of arm $i$ can be interpreted as a Brownian motion with drift, where the increment is a normal random variable $N(\mu_i \Delta t,\sigma^2 \Delta t)$, and the allocated resource $A_{ti}$ specifies the time horizon of the Brownian motion. Note that the Brownian motion can be normalized within the unit time horizon using the self-similar property $B(a) \overset{d}{=} \sqrt{a} B(1)$, which implies that 
	after a time duration $a$, the distribution of the Brownian motion is equivalent to multiplying the Brownian motion observed at the unit time by $\sqrt{a}$.
	%after proceeding time $a$, the Brownian motion has the same distribution as $\sqrt{a}$ multiplied by the Brownian motion proceeding with the unit time.
	}
	
	\RV{However, in practical applications,  particularly in financial markets, the increments of the reward may exhibit correlation, which is referred to as Fractional Brownian motion, or are heavy-tailed distributed, known as the stable Levy process. The Hurst parameter $H \in [0,1]$ (see Chapter 4 of \citet{shiryaev1999essentials}) is employed to capture the correlation of increments or the decay rate of the heavy-tailed increments. Both of the models adhere to the self-similar property $B(a) \overset{d}{=} a^H B(1)$, where the scale is generalized to $a^H$ from $a^{1/2}$. In line with this notion, we adopt the parameter $b$ to specify the magnitude of random fluctuations.}
	
	We further remark on the interpretation of $b$ in \eqref{equ:def-feedback}.
	Since $b>0$, the scale of noise $\tA^b$ increases in $\tA$.
	This makes sense, as allocating more resources to an arm results in more randomness.
	%This is natural, as more resource allocated to an arm leads to more randomness.
	On the other hand, because $b<1$, the \emph{normalized} scale of the noise $\tA^b/\tA=\tA^{b-1}$ is decreasing and convex in $\tA$.
	It implies that the information acquisition is convex, 
	as allocating more resources to an arm allows for more information to be aquired than if the resources were divided evenly.
	%as more resource allocated to an arm allows more information to be retrieved than the scenario when the resource is split. %\lwh{Lwh: Maybe it requires some basic knowledge about information acquisition for the reader to understand the above sentence. What's the meaning of convex information acquisition?}
	Moreover, because $\tA\le 1$, a larger $b$ represents a less noisy environment.
	So the value of $b$ reflects the signal-to-noise ratio and characterizes how the scale of noise $\tA^b$ changes with the allocated resource $\tA$.
	
	\RV{\textbf{Connection with MAB Literature}.} The constant $b$ also enables our formulation to be connected to \emph{the standard stochastic multi-armed bandit (SMAB) and online learning with full feedback}.
	When $b\to 1$, even if $\tA\ll 1$ is small, the decision maker can observe $\tY/\tA=\mu_i+\xi_{ti}$, which is as good as $\tA=1$.
	Therefore, by allocating infinitesimal resources to all the arms, the decision maker can obtain the full feedback \citep{degenne2016anytime,slivkins2019introduction}.
	Conversely, when $b\to 0$, the signal-to-noise ratio becomes very small and the normalized scale $\tA^{-1}$ becomes the most convex.
	In this case, any $0<\tA<1$ is suboptimal in terms of exploration, and the optimal regret matches that of SMAB.
	Therefore, the policy that attains the optimal regret allocates binary $\tA$s (see Section \ref{sec:gap-indep-regret-low}).  
	It's worth noting that when $\tA\in\{0,1\}$, our formulation is equivalent to SMAB and $b$ no longer plays a role.

	 \RV{Because of the connection to SMAB, the regret bound of SMAB can provide benchmarks for our problem.
	 In particular, if we restrict the policy space to be $\{\tA: \sum_{i=1}^K \tA=1, \tA \in \{0,1\}\}$, then our problem is identical to SMAB.
	 Since the policy space is larger, i.e., $\tA \in [0,1]$ instead of $\{0,1\}$, the regret bound of SMAB serves as a lower bound for our problem.
	 Specifically, it is well known that (see, e.g., \citealt{lattimore2020bandit}) the optimal gap-independent regret bound  for SMAB is $O(\sqrt{KT})$,	
	 and the optimal gap-dependent regret bound is $O\left(\sum_{i=1}^K \Delta_i+\log T/\Delta_i\right)$ where $\Delta_i$ is the suboptimality gap of arm $i$.
	 In the next two sections, we analyze the regret bounds of our problem in both scenarios and show when the regret bounds of SMAB are tight for our problem.}
	
	 \RV{Before getting into the analysis, we provide a toy example to shed light on the design of the algorithm. 
	 The key tradeoff in the algorithm is whether to pool the resource on a single arm or evenly divide the resource among multiple arms in one period.
	 In the toy examples, we consider $K$ arms over $K$ periods. Suppose that the noise in \eqref{equ:def-feedback} follows a standard normal distribution. 
	 We compare two simple allocation strategies: exploring one arm in turn in a period or dividing the resource evenly among all the arms in a period.
	 For the former strategy, it is easy to see that the estimator for the mean reward $\hat{\mu}_{1i}-\mu_i \sim N(0,1)$, as it is based on a single sample from the normal distribution.	
	 If we choose to divide, the allocated resources $A_{1i}=\ldots=A_{Ki}=1/K$ and the mean estimator $\hat{\mu}_{Ki}-\mu_i= \sum_{t=1}^K \xi_{ti}/K^b \sim N(0,K^{1-2b})$. 
	 The allocation strategy with a smaller variance is preferable. 
	 Comparing $K^{1-2b}$ with $1$, we find when $b>1/2$ ($<1/2$), it is better to divide (pool) the resource.
	 This toy example illustrates the role of $b$, why phase transition happens at $b=1/2$, and that the algorithms should be designed differently for $b>1/2$ and $b<1/2$.
	 This intuition guides us through the analysis in the next two sections.}

	\section{Gap-independent Regret Bound}\label{sec:gap-indep-regret}

	\subsection{Lower Bound} \label{sec:gap-indep-regret-low}
	We first show the lower bound for the regret of the gap-independent case.  
	\begin{theorem}\label{thm:gap-ind-low}
		Suppose $T \ge K$ and $\xi_{ti} \sim N(0,1)$. Then we have the gap-independent lower bound
		\begin{equation*}
			R_\pi(T) \ge \left\{
			\begin{array}{lc}
				\frac{1}{27} \sqrt{(K-1)T}, &  \text{for} \ 0 \leq b \le 1/2, \\
				\frac{1}{27} (K-1)^{1-b}\sqrt{T} &  \text{for} \ 1/2 \le b \le 1.
			\end{array}
			\right.
		\end{equation*}
	\end{theorem}
	According to Theorem \ref{thm:gap-ind-low}, when $b \le 1/2$, the regret lower bound $O(\sqrt{KT})$ is equivalent to the SMAB problem.
	This implies that the standard algorithms such as UCB can be used by forcing $\tA\in \{0,1\}$, resulting in an optimal regret.
	%It implies we can simply use the standard algorithms such as UCB by forcing the allocated resource $\tA\in \{0,1\}$, and the optimal regret can be achieved.
	Hence, Theorem \ref{thm:gap-ind-low} provides both an algorithm and a tight upper bound for the gap-independent case when $b\le 1/2$. 
	%Note that it automatically gives us an algorithm and a tight upper bound when $b\le 1/2$ in the gap-independent case.
	On the other hand, when $b>1/2$, the lower bound  $O(K^{1-b} \sqrt{T})$ is smaller than $O(\sqrt{KT})$, suggesting that the standard SMAB algorithms can be improved.
	
	The proof of Theorem \ref{thm:gap-ind-low} mainly follows from the Le Cam's method \citep{lecam1973convergence,yu1997assouad}.
	% that the lower bound depends on how hard to differentiate the two bandit instances which measure by the KL-divergence of the sample path. 
	% It's easier to differentiate the instances with larger KL divergence. 
	The major difference from the standard analysis in SMAB is that the KL-divergence depends on $B_{Ti} \coloneqq \sum_{t=1}^T \tA^{2-2b}$ instead of the ``total number of pulls'', $S_{Ti} \coloneqq \sum_{t=1}^T \tA$. 
	To see this, recall that conditional on $\tA$, the random reward for arm $i$ is $\tY(\tA)=\tA \mu_i+\tA^b \xi_{ti} \sim N(\tA \mu_i, \tA^{2b}) \coloneqq P_{A_{ti}}$. Considering another problem instance where the mean reward of arm $i$ is $\mu_{i}'$, the random reward follows the distribution $P'_{A_{ti}}=N(\tA \mu_i', \tA^{2b})$. The KL-divergence between $P_{A_{ti}}$ and $P'_{A_{ti}}$ is 
	\begin{equation*}
		KL(P_{A_{ti}}, P'_{A_{ti}})=A_{ti}^2 (\mu_{i}-\mu_{i}')^2/(2 A_{ti}^{2b})=A_{ti}^{2-2b} (\mu_{i}-\mu_{i}')^2/2 =A_{ti}^{2-2b} KL(P_i,P_i'),
	\end{equation*}
	where the two distributions $P_i \sim N(\mu_i,1), P_{i'} \sim N(\mu_i',1)$.  
	Therefore, the KL-divergence depends on $B_{Ti}$ instead of $S_{Ti}$.
	
	To see why there is a phase transition at $b=1/2$ in the lower bound, we compare $B_{Ti}$ with $S_{Ti}$. 
	When $b \le 1/2$, we have $\tA^{2-2b} \le \tA$ and $B_{Ti} \le S_{Ti}$, which implies that the KL-divergence of the sample path will be no larger in our problem than in SMAB. 
	However, when $b > 1/2$, we have $\tA^{2-2b} \ge \tA$ and $B_{Ti} \ge S_{Ti}$, which allows us to obtain a smaller lower bound than SMAB due to the increased KL-divergence.
	
	\subsection{Algorithm for the Gap-independent Case}\label{sec:alg-ind}
	Next, we design algorithms that achieve the regret lower bound in Theorem~\ref{thm:gap-ind-low}.
	Note that when $b \le 1/2$, we have shown that the regret lower bound matches that of SMAB. 
	So standard SMAB algorithms such as UCB already achieve the lower bound in our problem. 
	In this section, we focus on the bound $O(K^{1-b} \sqrt{T})$ when $b \in \left(0.5,1\right]$.
	
	Before we present the algorithm, it's worth noting that there are multiple estimators for the mean reward estimation in our problem, as opposed to the sample mean used in SMAB.
	
	\textbf{Mean Estimators.}
	Recalling the definition of the random reward $\tY$ in \eqref{equ:def-feedback}, the first estimator sums up the random rewards and then divides it by the sum of resources: %\lwh{LWH: After checking the dictionary, I find the word ``resource'' in our paper should be used in the plural form. Should I revise them?}
	\begin{equation}\label{equ:mean-one}
		\muone=\frac{\sumst \sY}{\sumst \sA}=\mu_i+\dfrac{\sumst \sA^{b} \xi_{si}}{\sumst \sA}.
	\end{equation}
	% If $\sumst \sA=0$, there are no samples observed for arm $i$, thus no need to estimate the mean.
	The second mean estimator divides $\tY$ by $\tA$ and then sums up all the normalized rewards:
	\begin{equation}\label{equ:mean-two}
		\mutwo=\frac{1}{\sum_{s=1}^t \Ind (\sA>0)} \sum_{s=1}^t \frac{\sY}{\sA} \Ind(\sA>0)=\mu_i+\frac{1}{\sum_{s=1}^t \Ind (\sA>0)} \sum_{s=1}^t \sA^{b-1} \xi_{si} \Ind(\sA>0).
	\end{equation}
	It is easy to see that both estimators are unbiased if $\{A_{si}\}_{s=1}^t$ were a deterministic sequence.
	\RV{Note that in most cases, $\muone \neq \mutwo$ for the same history $H_t$. For example, considering $t=2$ and $A_{1i}=1/2$, $A_{2i}=1/3$, $Y_{1i}=2$, $Y_{2i}=3$, we have $\mutwo=6$ and $\muone=6.5$. However, when $\tA$ is restricted to $\{0,1\}$, the two estimators are equivalent, i.e,
	 \begin{equation}\label{equ:SMAB-mean}
		 \muone=\mutwo=\mu_i+\dfrac{\sumst \xi_{si} \Ind(\sA=1)}{\sumst \Ind (\sA=1)} \coloneqq \hat{\mu}_{ti},
		 \end{equation} 
	 which is the sample mean in SMAB.}
	
	% \nc{filtration or history? we should stick to one notation.}
	%	 \RV{Note that the second estimator $\mutwo$ is a martingale process with the expectation $\EX[\mutwo]=\mu_i$. To see this, conditional on the past history up to time $t$, i.e., $H_t$, we have
	%	 \begin{equation*}
	%		 \EX[\mutwo|H_t]=\mu_i+\frac{1}{\sum_{s=1}^t \Ind (\sA>0)} \left(\sum_{s=1}^{t-1} \sA^{b-1} \xi_{si} \Ind(\sA>0)+\sA^{b-1} \Ind(\sA>0) \EX[\xi_{si}]\right)=\hat{\mu}_{t-1,i}^{(2)},
	%		 \end{equation*}
	%	 and 
	%	 \begin{equation*}
	%		 \EX[\mutwo]=\EX\left[\EX\left[\mutwo|H_t\right]\right]=\EX[\hat{\mu}_{t-1,i}^{(2)}]=\ldots=\EX[\hat{\mu}_{1i}^{(2)}]=\mu_i.
	%		 \end{equation*}
	%	 However, the first estimator \eqref{equ:mean-one} is not a martingale since $\sumst \sA$ appears in the denominator, where $\sA$ depends on the past history $H_{s}$. So the property of $\muone$ depends on the algorithm we design. }
	
	We investigate the standard error of the estimators, which is crucial in the construction of confidence intervals.
	Note that $\muone$ and $\mutwo$ have different variances even for a given deterministic sequence $\{A_{si}\}_{s=1}^t$. 
	To see this, suppose the noise $\xi_{si} \sim N(0,1)$ and denote the total number of pulls $\tS = \sumst \sA$. 
	In the sample mean used in SMAB, we have $\hat{\mu}_{ti}-\mu_i \sim N(0,1/S_{ti})$. However, for $\muone$, we have
	\begin{equation}\label{equ:def-R2}
		\muone-\mu_i \sim N(0,1/\tRone), \ \text{where} \  \tRone=\dfrac{\left(\sumst \sA\right)^2}{\sumst \sA^{2b}}.
	\end{equation}
	For $\mutwo$, we have
	\begin{equation}\label{equ:def-R1}
		\mutwo-\mu_i \sim N(0,1/\tRtwo), \ \text{where} \  \tRtwo=\dfrac{\left(\sum_{s=1}^t \Ind (\sA>0)\right)^2}{\sumst \sA^{2b-2} \Ind (\sA>0)}.
	\end{equation} 
	Thus $\tRone$ and $\tRtwo$ affect the variance of $\muone$ and $\mutwo$. 
	If $\tRone$ and $\tRtwo$ are larger than $\tS$, then the mean estimators $\muone$  and $\mutwo$ converge to $\mu_i$ faster than $\hat{\mu}_{ti}$ and thus reduce the regret.
	% The new estimators will create space for the reduction of SMAB regret. 
	% If we can design the algorithm such that $\tRtwo (\tRone) \gg \tS$, we can use the same quantity for exploration, but achieve more accurate estimation. 
	We use $\muone$ and $\mutwo$ to design the algorithms for matching the gap-independent and gap-dependent bound respectively.
	 \RV{\begin{remark}[Alternative Mean Estimators]\label{remark:alter-estimator}
		 We point out that there are other alternative estimators. We divide the time periods from $1$ to $t$ into $L$ intervals: $[1,t_1],[t_1+1,t_2],\ldots,[t_{L-1}+1,t]$ and define  $t_0=0,t_L=t$.
		 In each interval, we give an estimator using \eqref{equ:mean-one} and then summarize all the estimators using \eqref{equ:mean-two}:
		 \begin{equation}
			 \muthree=\dfrac{1}{\sum_{l=1}^L \Ind \left(\sum_{s=t_{l-1}+1}^{t_l} \sA>0\right)} \left(\sum_{l=1}^L \dfrac{\sum_{s=t_{l-1}+1}^{t_l} \sY}{\sum_{s=t_{l-1}+1}^{t_l} \sA} \Ind \left(\sum_{s=t_{l-1}+1}^{t_l} \sA>0\right) \right).
			 \end{equation}
		 The estimators $\muone$ and $\mutwo$ are special cases of $\muthree$ where $L=1$ and $t$. The total number of possible $\muthree$ is $2^t$ because every period can be the end of an interval. So the choice of intervals will affect the value of $\muthree$, and makes the analysis complicated. In this paper, we use $\muone$ and $\mutwo$ to design algorithms. We leave the analysis of $\muthree$ as a future research direction.
		\end{remark}}
	
	\textbf{Algorithm.}
	The algorithm we propose for the gap-independent case is based on successive elimination (SE). 
	In the first period, the algorithm uniformly allocates the resource to all arms. 
	Over time, it gradually eliminates the suboptimal arms with low mean reward and keeps track of the active set of arms.
	The resource is uniformly allocated to the arms in the active set.
	The elimination rule for the arms is the critical point. We construct a confidence interval using the mean estimator \eqref{equ:mean-one} and compare the upper confidence bound (UCB) of an active arm to the lower confidence bound (LCB) of the current optimal arm. 
	If the UCB is less than the LCB, then we expect the arm is not optimal and should be eliminated. 
	Finally, at the end of $T$ periods, the suboptimal arms will most likely be eliminated, leaving only the true optimal arm. 
	%We expect that the regret incurred by selecting the suboptimal arms will not be too large. 
	The detailed steps of the algorithm are shown in Algorithm \ref{alg:SE}.
	
	%	\begin{algorithm}[htbp]
		%		\caption{Successive Elimination with Divisible Resources}
		%		\label{alg:SE}
		%		%\SetAlgoNoLine
		%		\KwIn{constant $b \in \left[1/2,1\right),\sigma$}
		%		Initialize the active set $D\gets [K]$ \\
		%		\For{$t=1,2,\ldots,T$}{
			%			Allocate $1/|D|$ to each arm in $D$ and observe $Y_{ti}$ for $i \in D$ \\
			%			Estimate $\tRone=\left(\sumst \sA\right)^2 \big/\left(\sumst \sA^{2b}\right)$, \ 
			%			$CI_{ti}=3 \sigma \log t \bigg/\sqrt{\tRone}$ \\
			%			 Construct $UCB_{ti}=\muone+CI_{ti}$, \  $LCB_{ti}=\muone-CI_{ti}$ \\
			%		For $i \in D$, if $UCB_{ti} < \max_{j \in D_{t-1}} LCB_{tj}$, then $D \gets D\setminus\{i\}$. 
			%		}
		%	\end{algorithm}
	
	\begin{algorithm}[htbp]
		\caption{Successive Elimination with Divisible Resources}
		\label{alg:SE}
		\begin{algorithmic}[1]
			\STATE \textbf{Input}: constant $b \in \left[1/2,1\right),\sigma$
			\STATE Initialize the active set $D\gets [K]$
			\FOR{$t=1,2,\ldots,T$}
			\STATE Allocate $1/|D|$ to each arm in $D$ and observe $Y_{ti}$ for $i \in D$
			\STATE Let $\tRone=\left(\sumst \sA\right)^2 \big/\left(\sumst \sA^{2b}\right)$, \ 
			$CI(t,i)=2 \sqrt{3} \sigma \sqrt{\log t \log T} \bigg/\sqrt{\tRone}$ 
			\STATE Construct $UCB_{ti}=\muone+CI(t,i)$, \  $LCB_{ti}=\muone-CI(t,i)$
			\STATE For $i \in D$, if $UCB_{ti} < \max_{j \in D} LCB_{tj}$, then $D \gets D\setminus\{i\}$
			\ENDFOR
		\end{algorithmic}
	\end{algorithm}
	
%	The main difference between Algorithm \ref{alg:SE} and the standard SE algorithm for SMAB \citep{even2002pac,even2006action,perchet2013multi,gao2019batched} is the construction of the confidence interval.
	\RV{Compared with the standard SE algorithm for SMAB \citep{even2002pac,even2006action,perchet2013multi,gao2019batched}, the major difference is the construction of the confidence interval.
	 First, only one arm is allowed to select in standard SE, while Algorithm~\ref{alg:SE} explores all the surviving arms in one period by dividing the unit resource, which accelerates the exploration rate.} 
	Second, the confidence interval is $CI_{ti}=O(\sqrt{\log T/\tS})$ in standard SE, whereas in Algorithm~\ref{alg:SE}, it's $O(\log T\big/\sqrt{\tRone})$. 
	As showed in \eqref{equ:def-R2}, the variance of $\muone$ is determined by the sequence $\tRone$ rather than $\tS$. 
	\RV{Also, the $CI_{ti}$ grows in $O(\log t)$ instead of $O(\sqrt{\log t})$. 
	This is because a larger confidence interval is required to make sure the validity of concentration inequality for $\muone$ than the mean estimator for SMAB showed in \eqref{equ:SMAB-mean}.}
	We also remark that Algorithm \ref{alg:SE} is straightforward to implement: it only requires the knowledge of constant $b$ and sub-Gaussian parameter $\sigma$, and the quantities $\muone,\tRone$ can be computed efficiently.

	\begin{theorem}\label{thm:SE-indep-up}
		For $T \ge K$, Algorithm~\ref{alg:SE} achieves the gap-independent regret $O(K^{1-b} \sqrt{T} \log T)$.
	\end{theorem}
	The regret upper bound in Theorem~\ref{thm:SE-indep-up} matches the lower bound in Theorem~\ref{thm:gap-ind-low}, ignoring the $\log T$ term. So the Algorithm \ref{alg:SE} achieves the optimal rate of regret.

	%  \nc{the next paragraph needs to be written better and more carefully.}
	%  \lwh{For general sub-Gaussian distribution, we can only estimate a lower bound for $b$. To give a precise estimation, we should specify the distribution like Gaussian. Maybe we delete the the next remark in the COLT version.}
	%	\begin{remark}[Estimate $b$]
	%		If the decision maker knows $\sigma$ but does not know $b$ at the beginning, he can use a burning stage with $2L$ periods to estimate it. More precisely, fix arm $i$ and the allocated resources $A_{1i}=A_{2i}=\ldots=A_{2L,i}=a$, and observe the rewards $\{\tY\}_{t=1}^{2L}$. To offset the mean effect, we compare the rewards in two consecutive periods:
	%		\begin{equation*}
	%			D_{t,i} \coloneqq Y_{2t-1,i}-Y_{2t,i}=a^b (\xi_{2t-1,i}-\xi_{2t,i}) \sim \sqrt{2} a^b \sigma\text{-sub-Gaussian}.
	%		\end{equation*}
	%		According to the property of sub-Gaussian distribution, we have $\EX(D_{t,i}^2) \le 2 a^{2b} \sigma^2$. Then, we use the empirical estimate $\frac{1}{L} \sum_{t=1}^L D_{t,i}^2$ to approximate $\EX(D_{t,i}^2)$, and give a lower bound to $b$:
	%		\begin{equation*}
	%			b \ge \frac{1}{2 \log a} \left(\log\left(\sum_{t=1}^L D_{t,i}^2\right)-\log(2 \sigma^2 L)\right).
	%		\end{equation*}
	%	\end{remark}
	\RV{\begin{remark}[Estimate $b$]
		When $b$ is unknown, we can dedicate a burning period in the beginning of the horizon to estimate it. More precisely, we fix arm $i$ and allocate $A_1=A_{L+1}=a_1$, $ A_2=A_{L+2}=a_2$, …, $A_L=A_{2L}=a_L$, and observe the rewards $Y_t, t=1,\ldots,{2L}$. To offset the mean effect, we compare $Y_t$ with $Y_{t+L}$ and let $D_t \coloneqq Y_{t+L}-Y_t=a_t^b (\xi_{t+L}-\xi_t)$ which follows the distribution of $\sqrt{2} a_t^b \sigma$ sub-Gaussian. Taking logarithm of both sides, we have $\log D_t= b \log a_t+ \log (\sqrt{2} \xi_t)$. Then we can estimate $b$ by regressing $\log D_t$ on $\log a_t$. Although we have not provided a rigorous proof, this procedure is unlikely to increase the regret rate. 
		%We will provide a more rigorous treatment in the revision.
		Moreover, if the decision maker knows a lower bound of b, i.e., $b_0 \in [0.5, b]$, then he can use $b_0$ in the algorithms to achieve the regret $O(\sqrt{T} K^{1-b_0})$ %for the gap-independent case and  $O(\sum \Delta_i^{1-2/(2 b_0 -1)})$ for the gap-dependent case. 
		That’s because the misspecified $b$ only enlarges the confidence interval constructed in Algorithm \ref{alg:SE},
		%and increases the exploration effort in Algorithm 2, 
		but will not affect the convergence rate of the mean estimator $\hat{\mu}_{ti}^{(1)}$ and $\hat{\mu}_{ti}^{(2)}$.
	\end{remark}}

	\subsection{Analysis of Theorem \ref{thm:SE-indep-up}}
	
	%In this subsection, we provide the analysis for Theorem \ref{thm:SE-indep-up}.
	
	\textbf{Process Represented by Stopping Times.} 
	As we mentioned in Section~\ref{sec:formu}, the decision space of our problem in one period is the simplex in $[0,1]^K$, much harder to track than SMAB.
	% So the whole desicion space for $T$ periods is $[0,1]^{TK}$. 
	The benefit of SE is to reduce the decision space to $K$ stopping times. 
	To see this, let $\{\tau_i\}_{i=1}^K \in \{1,2,\ldots,T\}$ be the last period arm $i$ remains active ($T$ for arms surviving to the last period).
	Let $\tau_{(1)} \le \tau_{(2)} \le \ldots \le \tau_{(K)}$ be their order statistics. 
	In period $t \in [1,\tau_{(1)}]$, all $K$ arms are active. 
	For $t \in [\tau_{(1)}+1,\tau_{(2)}]$, one arm is eliminated and $K-1$ arms remain active, and so on.
	From Algorithm \ref{alg:SE}, at least one arm survives in period $T$ so $\tau_{(K)}=T$.
	% If $\tau_{(K-1)} < T$, only one arm remains active for $t \in [\tau_{(K-1)}+1,T]$. Otherwise, multiple arms survive at period $T$. 
	The resource allocated to arm $i$ at period $t$ is
	\begin{equation}\label{equ:inden-SE-A-tau}
		A_{ti}=\frac{1}{K-j+1} \Ind \left(\tau_{(j-1)}+1 \le t \le \tau_{(j)} \le \tau_i\right),
	\end{equation}
	where $\tau_{(0)}\coloneqq 0$. 
	% We illustrate the stopping times in Figure~\ref{fig:Stopping_time}.  
	To analyze the algorithm, we can focus on the smaller space of the stopping times $\bm{\tau}=\{\tau_1,\tau_2,\ldots, \tau_K\} \in \{1,2,\ldots,T\}^K$,
	because of the one-to-one correspondence between $\bm{\tau}$ and the regret.
	% $D_t$, i.e., $i \in D_t \Leftrightarrow \tau_i \ge t$ for arm $i$. 
	% For example, consider $K=3$, $t=3$, $\tau_1=1$, $\tau_2=\tau_3=2$, $\tau_4=3$, then $A_{11}=A_{12}=A_{13}=A_{14}=1/4$; $A_{12}=0, A_{22}=A_{32}=A_{42}=1/3$, $A_{14}=A_{24}=A_{34}=0$, $A_{44}=1$. 
	\begin{figure}
		     \centering
		     \includegraphics[width=0.7\textwidth]{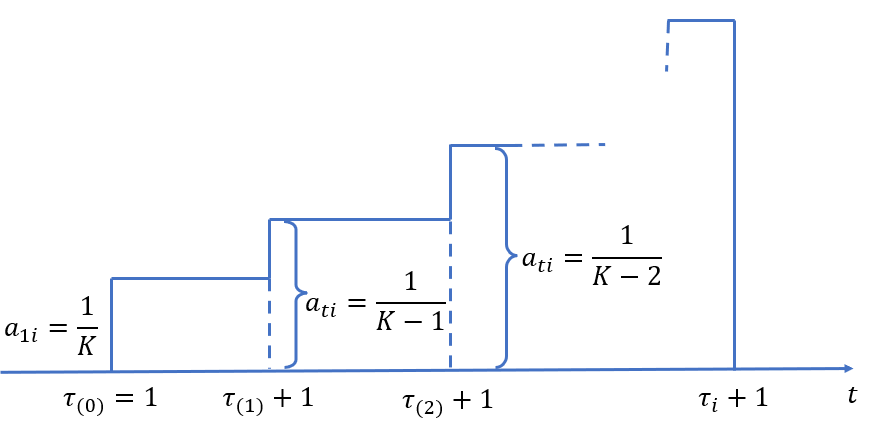}
		     \caption{Relationship between $A_{ti}$ and $\bm{\tau}$.}
		     \label{fig:Stopping_time}
	\end{figure}
	
	\textbf{A New Concentration Inequality.} The stopping times are random variables whose realizations depend on the problem instance 
	% $v(\mu_1,\mu_2,\ldots,\mu_K)$ 
	and the noise $\{\{\xi_{ti}\}_{i=1}^K\}_{t=1}^T$. 
	It creates challenges for the concentration inequality of $\muone$ showed in \eqref{equ:mean-one},
	which has not appeared in the analysis of SMAB algorithms. More precisely, in the standard concentration result, for independent noise $\{\xi_{si}\}_{s=1}^t$, the sample mean $\sum_{s=1}^t \xi_{si}/t$ grows at the rate $O(\sqrt{\log t/t})$ with a high probability. 
	However, for the mean estimator $\muone$, the allocation $\tA$ in the denominator is also correlated with the noises $\xi_{1i},\ldots,\xi_{t-1,i}$. 
	%So conditional on $\{\sA\}_{s=1}^t$, the noises $\xi_{1i},\xi_{2i},\ldots,\xi_{t-1,i}$ are not independent and the standard concentration results will not hold for $\muone$. 
	As a result, the sequence $\muone-\mu_i= \sum_{s=1}^t \sA^b \xi_{si} \big/ \sumst \sA$ cannot be expressed as a sum of martingale differences, and results such as Azuma's or Bernstein's inequalities can not be applied. 
	 To see this, note that
	 	\begin{equation*}
		 	\EX\left[\muone-\mu_i|H_t\right]=\frac{\sum_{s=1}^{t-1} \sA^{b} \xi_{si}}{\sumst \sA}+\frac{\tA^b}{\sumst \sA} \EX[\xi_{ti}]=\frac{\sum_{s=1}^{t-1} \sA^{b} \xi_{si}}{\sumst \sA} \neq \frac{\sum_{s=1}^{t-1} \sA^{b} \xi_{si}}{\sum_{s=1}^{t-1} \sA} = \hat{\mu}_{ti}^{(1)}-\mu_i.
		 	\end{equation*}  
		So the dependence between $\bm{A}_i$ and $\bm{\xi}_i$ makes it difficult to prove the concentration for $\muone$.
	 \begin{remark}[Concentration Inequalities Used in SMAB]
		 If we restrict $\tA$ to zero or one, then $\muone-\mu_i$ reduces to the sample mean $\sumst \xi_{si} \Ind(\sA=1)/ \tS$ for SMAB. 
		 Although the same issue exists ($\tS$ is random),
		 it is common to use a union bound for all $\tS \in \{1,2,\ldots,T\}$ and the concentration probability is typically still well-controlled after being inflated by a factor of $T$. 
		 In our problem, however, since the allocation is continuous, the union bound blows up the concentration probability.
		 This is why we need to develop new approaches in this paper.
		 \end{remark}
	 
	To address this issue, we decouple the dependence between $\bm{A}_i$ and $\bm{\xi}_i$ by considering the worst-case scenario:
	% Note that the right-hand side of \eqref{equ:gap-ind-sup} does not match with the right scale of convergence. For example, considering fixed sequence $A_{1i}=A_{2i}=\ldots=A_{ti}=1/K$, then $X_t=t^{-1} K^{1-b}\sumst \xi_{si}$ scales as $K$ and $t$. So we normalize it by multiplying $\sqrt{\tRone}$:
	\begin{equation}\label{equ:gap-ind-sup-nor}
		\left|(\muone-\mu_i)\sqrt{\tRone}\right| \le \sup_{\bm{A}_i \in \mathcal{A}} \frac{\sumst \sA^b \xi_{si}}{\sqrt{\sumst \sA^{2b}}}.
	\end{equation}
	where we have normalized the deviation by the standard deviation in \eqref{equ:def-R2}.
	% \begin{equation}\label{equ:gap-ind-sup}
		% \left|\muone-\mu_i\right| \le \sup_{\bm{A}_i \in \mathcal{A}} \frac{\sumst \sA^b \xi_{si}}{\sumst \sA},
		% \end{equation}
	Here $\mathcal{A}$ is the set of all the possible sequences $\bm{A}_i$ from Algorithm \ref{alg:SE} before the elimination of arm $i$, i.e., 
	$\mathcal{A} =\{\left(A_{1i},\ldots,A_{ti}\right): A_{si} \in \{1/K,1/(K-1),\ldots,1\}, \  A_{1i} \le \ldots \le A_{ti}\}$.
	As we note in \eqref{equ:inden-SE-A-tau}, the sequence $\bm{A}_i$ is uniquely determined by the stopping times $\bm{\tau}$. 
	%Also, arm $i$ should be active at period $t$. Otherwise, it's no need to compute $A_{ti}$. 
	% By the design of Algorithm~\ref{alg:SE}, for $t \le \tau_i$, $\{A_{1i}, A_{2i},\ldots, A_{ti}\}$ is monotonically increasing and taking values in $\{1/K,1/(K-1),\ldots,1\}$. 
	
	% The supreme operator turns out to be the major challenge. 
	% then the quantity cannot be bounded properly.
	% Without the supreme, \eqref{equ:gap-ind-sup-nor} has the same scale with a $\sigma$-sub-Gaussian random variable. To see this, for decoupled $\bm{A}_i$ and $\bm{\xi}_i$, the noise $\xi_{si}$ are i.i.d. $\sigma$-sub-Gaussian and $\sumst \sA^b \xi_{si}$ is $\left(\sumst \sA^{2b}\right) \sigma$-sub-Gaussian. Thus, $\sumst \sA^b \xi_{si} \big/ \sqrt{\sumst \sA^{2b}}$ is $\sigma$-sub-Gaussian.
	% Because of the $\sup$, the \eqref{equ:gap-ind-sup-nor} may not be sub-Gaussian. That's because the sequence $\bm{A}_i$ is chosen after observing $\bm{\xi}_i$. It can be chosen close to the direction of $\bm{\xi}_i$ to maximize \eqref{equ:gap-ind-sup-nor}. 
	It's challenging to give a concentration result for the RHS of \eqref{equ:gap-ind-sup-nor}.  
	Note that the cardinality of $\mathcal{A}$ is in the same order as $ {t+K \choose K} \approx t^K$ if $t \gg K$. 
	Therefore, applying the union bound to $\mathcal A$ introduces an exponential term in $K$ that prohibits us from obtaining the optimal regret.
	Another common technique used in nonparametric statistics to bound such an event is to use covering numbers: decompose $\mathcal A$ into a number of balls and apply the union bound to the ball centers; use continuity to bound the points in each ball.
	However, the covering number of $\mathcal A$ is not much smaller than the covering number of $[0,1]^{t}$, equivalent to the removal of the monotonicity constraint.
	As explained in the introduction, it leads to a much worse concentration probability and we cannot obtain the optimal regret.

	%	Fortunately, $\bm{A}_i$ should belong to a monotonic set $\mathcal{A}$, which restricts the flexibility and prevents \eqref{equ:gap-ind-sup-nor} to be too large. The following proposition shows a concentration result for \eqref{equ:gap-ind-sup-nor} by taking advantage of the monotonic sequence set.
	Next, we state our major theoretical result in this section.
	We first introduce a continuous version of $\mathcal A$:
	$\mathcal{A}_m=\{(a_1,a_2,\ldots,a_t): 0<a_1 \le a_2 \le \ldots \le a_t \le 1\}$.
	\begin{proposition}[Concentration Inequality for Monotone Weights]\label{prop:concen-mono}
		%For a monotonic sequence set $\mathcal{A}_m=\{0<a_1 \le a_2 \le \ldots \le a_t \le 1\}$, i.i.d. $\sigma$-subgaussian $\xi_s$ and any constant $\epsilon>0$, we have
		For $\mathcal{A}_m$ and independent $\sigma$-sub-Gaussian $\xi_s$ and any constant $\epsilon>0$, we have
		%    \begin{equation}\label{equ:concen-mono-result}
			%    \PR\left( \sup_{\bm{a} \in \mathcal{A}_m}\frac{\sum_{s=1}^t a_{s} \xi_{s}}{\sqrt{\sum_{s=1}^t a_{s}^2}} \ge  \left(\frac{3}{2} \log t \right)^{1/2} \epsilon \right) \le t \exp\left(-\frac{\epsilon^2}{2 \sigma^2}\right).
			%    \end{equation}
		\begin{equation}\label{equ:concen-mono-result}
			\PR\left( \sup_{\bm{a} \in \mathcal{A}_m}\frac{\sum_{s=1}^t a_{s} \xi_{s}}{\sqrt{\sum_{s=1}^t a_{s}^2}} \ge  \sqrt{\frac{3}{2}} (\log t)  \epsilon \right) \le \exp\left(-\frac{\epsilon^2}{2 \sigma^2}\right).
		\end{equation}
	\end{proposition}
	In \eqref{equ:concen-mono-result}, taking $\epsilon=O(\sigma)$, we have that \eqref{equ:gap-ind-sup-nor}$=O(\log t)$ with high probability.
	%\PR(\eqref{equ:concen-mono-result} \ge 3\sigma \log t) \le t^{-2}$. 
	% So with a high probability, $\eqref{equ:concen-mono-result}=O(\log t)$. 
	To compare with the standard concentration inequalities, noting that for $a_1=a_2=\ldots=a_t$, one has $ (\sumst \xi_s)/\sqrt{t}=O(\sqrt{\log t})$. So taking the supremum over $\mathcal{A}$ doesn't inflate the concentration probability significantly, and the relaxation from $\mathcal{A}$  to $\mathcal{A}_m$ is not loose.
	Meanwhile, the monotonicity constraint is crucial.
	If we relax the set $\mathcal{A}$ to $[0,1]^t$, the growth rate of \eqref{equ:gap-ind-sup-nor} is $O(\sqrt{t})$, which is too loose for the regret. 
	To see this, let $c_{s} \coloneqq a_s \big/ \sqrt{\sumst a_s^{2}}$, we have 
	\begin{equation*}
		\EX\left(\sup_{\bm{a} \in [0,1]^t} \frac{\sumst a_s \xi_{s}}{\sqrt{\sumst a_s^{2}}}\right)=\EX\left(\sup_{\|\bm{c}\|_2=1, \bm{c} \succeq \bm{0}} \bm{c}^\top \bm{\xi} \right)
		\overset{(a)}{=}\EX\left(\sqrt{\sumst \xi_{s}^2 \Ind(\xi_{s} > 0)}\right)
		\overset{(b)}{=}
		O\left(\sqrt{t}\right),
	\end{equation*}
	%\begin{align*}
	%{}&\EX\left(\sup_{\bm{A}_i \in [0,1]^t} \frac{\sumst \sA^b \xi_{si}}{\sqrt{\sumst \sA^{2b}}}\right)=\EX\left(\sup_{\|\bm{B}_i\|_2=1, \bm{B}_i \succeq \bm{0}} \bm{B}_i^\top \bm{\xi}_i \right)	\\
	%&\overset{(a)}{=}\EX\left(\sqrt{\sumst \xi_{si}^2 \Ind(\xi_{si} > 0)}\right)
	%\overset{(b)}{=}O\left(\EX\left(\|\bm{\xi}_i\|_2\right)\right)\overset{(c)}{=}O\left(\sqrt{t}\right),
	%\end{align*}
	where $(a)$ follows from the optimal $c_{s}=\xi_{s} \Ind(\xi_{s}>0) \big/ \sqrt{\sumst \xi_{s}^2 \Ind(\xi_{s}>0)}$,  $(b)$ follows from $\EX\left(\sumst \xi_{s}^2 \Ind(\xi_{s} > 0)\right)=O\left(\EX\left(\sumst \xi_{s}^2 \Ind(\xi_{s} \le 0)\right)\right)=O\left(\EX\left(\|\bm{\xi}\|_2^2\right)\right)=O\left(\sqrt{t}\right)$, where the last equality holds by the concentration of the norm (see Theorem 3.1.1 in \cite{vershynin2018high} for detail). 
	
	\textbf{Sketch of the Proof of Proposition \ref{prop:concen-mono}.} 
	To utilize the monotonic property of sequence $\bm{a}$, we transform it to non-negative sequence $\bm{b}=(b_1=a_1, b_2=a_2-a_1, \ldots, b_t=a_t-a_{t-1})$. 
	Correspondingly, we transform the random variables $\{\xi_s\}_{s=1}^t$ to $\{f_s\}_{s=1}^t$, where $f_s=\xi_s+\ldots+\xi_t$. Thus, we have $\sum_{s=1}^t a_{s} \xi_{s}=\sum_{s=1}^t b_s f_s$ and
	\begin{equation}\label{equ:gap-indep-sup-ab-main}
		\sup_{\bm{a} \in \mathcal{A}_m}\frac{\sum_{s=1}^t a_{s} \xi_{s}}{\sqrt{\sum_{s=1}^t a_{s}^2}} = \sup_{\bm{b} \in \mathcal{B}} \dfrac{\sum_{s=1}^t b_s f_s}{\sqrt{\sum_{s=1}^t \left(\sum_{i=1}^s b_i\right)^2}},
	\end{equation}
	where $\mathcal{B} \coloneqq \{b_1>0, b_2 \ge 0 \ldots b_t \ge 0\}$. 
	After the transform, 
	% we can give a high-probability upper bound $\epsilon_s$ for $f_s$ and we have $\sumst b_s f_s \le \sumst b_s \epsilon_s$ due to $\bm{b} \succeq 0$. 
	since $f_s$ is $\sqrt{t-s+1} \sigma$-sub-Gaussian, we define $\epsilon_s \coloneqq \sqrt{t-s+1} \epsilon$ and show 
	\begin{equation}\label{equ:gap-indep-Gc-main}
		\PR(f_s \le \epsilon_s=\sqrt{t-s+1} \epsilon, \ \forall s) \ge 1-t\exp\left(-\frac{\epsilon^2}{2 \sigma^2}\right).
	\end{equation}
	Thus, with a high probability, we have
	\begin{equation}\label{equ:gap-indep-f-epsilon-main}
		\eqref{equ:gap-indep-sup-ab-main}^2=\sup_{\bm{b} \in \mathcal{B}} \dfrac{\left(\sum_{s=1}^t b_s f_s\right)^2}{\sum_{s=1}^t \left(\sum_{i=1}^s b_i\right)^2} \le \sup_{\bm{b} \in \mathcal{B}} \dfrac{\left(\sum_{s=1}^t b_s \sqrt{t-s+1}\right)^2}{\sum_{s=1}^t \left(\sum_{i=1}^s b_i\right)^2} \epsilon^2.
	\end{equation}
	The RHS of \eqref{equ:gap-indep-f-epsilon-main} has no randomness, and the numerator and denominator are quadratic functions of $\bm{b}$. 
	They can be written in the matrix form: $\left(\sum_{s=1}^t b_s \sqrt{t-s+1} \right)^2=\bm{b}^T A \bm{b}$, $\sum_{s=1}^t \left(\sum_{i=1}^s b_i\right)^2=\bm{b}^T B \bm{b}$, where $A=V V^T$, $V=(\sqrt{t},\sqrt{t-1},\ldots,1)^T$ and $B$ is positive definite according to the following decomposition:
	\begin{equation}\label{equ:indenp-matrix-D-main}
		B=D D^T,  \  D=D^T=\begin{pmatrix}
			1 & 1 & \ldots & 1 & 1\\
			1 & 1 & \ldots & 1 & 0\\
			\vdots & \vdots & \ldots & \vdots & \vdots\\
			1 & 0 & \vdots & 0 & 0
		\end{pmatrix}, \ 
		D^{-1}=\begin{pmatrix}
			0 & 0 & \ldots & 0 & 1\\
			0 & 0 & \ldots & 1 & -1\\
			\vdots & \vdots &\vdots & \vdots& \vdots\\
			1 & -1 & \ldots  & 0 & 0
		\end{pmatrix}.
	\end{equation}
	Based on the special structure of $A,B$, we apply the property of Schur complement to show that $\ A \prec 3\log(t) B/2$ and thus,  \eqref{equ:gap-indep-f-epsilon-main}$< 3\log(t)/2$. 
	Finally, we prove Proposition~\ref{prop:concen-mono} by combining \eqref{equ:gap-indep-sup-ab-main}, \eqref{equ:gap-indep-Gc-main}, and \eqref{equ:gap-indep-f-epsilon-main}.
	
	\RV{\begin{remark}[Other Attempts to Prove Proposition \ref{prop:concen-mono}]\label{remark:chaining}
		Note that the weighted average in \eqref{equ:concen-mono-result} can be reformulated as the supremum  over a continuous index set of $\bm{\xi}$: $\sup_{\bm{c} \in \mathcal{C}_m} \bm{c}^T \bm{\xi}=\sum_{s=1}^t c_s \xi_s$, where $c_s=a_s \big/ \sqrt{\sum_{s=1}^t a_{s}^2}$ and $\mathcal{C}_m = \{\bm{c}: \|\bm{c}\|_2=1, 0 < c_1 \le \ldots \le c_t \le 1\}$.  A standard technique to bound the supremum  over an index set is the chaining rule, such as the Dudley's chaining or generic chaining \citep{talagrand2022upper}. The fundamental idea behind this approach involves constructing a hierarchical partition tree such that for a specific index, there exists a path along the tree covering the index. By doing so, the supremum can be bounded by the length of the longest path of the tree. However, applying the chaining rule to prove a concentration result as \eqref{equ:concen-mono-result} poses challenges. That's because $\mathcal{C}_m$ is a special subset of the unit L2-sphere, and no existing results have been found to provide a suitable partition tree (or covering number) for $\mathcal{C}_m$. In the proof of Proposition \ref{prop:concen-mono}, instead of  construcing a partition tree geometrically, we adopt an algebraic approach to transform the monotonic index set to a non-negative set. This transformation significantly simplifies the subsequent analysis.
%		To give a concentration result for \eqref{equ:gap-ind-sup-nor}, a straight-forward idea is to use the union bound for each $\bm{A}_i \in \mathcal{A}$, since set $\mathcal{A}$ is discrete. For a fixed $\bm{A}_i$, $X_i$ is $\sigma$-sub-Gaussian. But the problem is the cardinality of $\mathcal{A}$ grows approximately polynomially in $t$. The error probability will be too large.
%		\begin{equation*}
%			\PR\left(\sup_{\bm{A}_i \in \mathcal{A}} \frac{\sumst \sA^b \xi_{si}}{\sqrt{\sumst \sA^{2b}}} \ge \epsilon\right) 
%			\le |\mathcal{A}| \PR\left( \frac{\sumst \sA^b \xi_{si}}{\sqrt{\sumst \sA^{2b}}} \ge \epsilon/|\mathcal{A}| \right) 
%			\le |\mathcal{A}| \exp\left(-\epsilon^2/(2 |\mathcal{A}|^2 \sigma^2)\right),
%		\end{equation*}
%		which goes to $1$ as $|\mathcal{A}|$ goes to $\infty$. 
%		
%		Another alternative is to use the idea of $\epsilon$ covering. We cover $\mathcal{A}_m$ by a set of small balls with radius $\epsilon$. Within a small ball, the sequences $\bm{A}_i$ are close to each other, and so does the $X_t$. Then, we choose a center of each small ball and take a union bound of all the centers. However, the total number of the balls is $O((1/\epsilon)^t/t!)$. It's not easy to choose the appropriate $\epsilon$ to balance the number of balls and the radius of each ball. 
	\end{remark}}

	\RV{\begin{remark}[Self-Normalized Martingale Process]
		In statistics, there's a stream of literature named self-normalized martingale process \citep{bercu2008exponential,victor2009self}. These works focus on bounding the sum of martingale differences by the sum of the conditional second moments. Since $\sumst \sA^b \xi_{si}$ is a martingale process, we can leverage the findings of these studies (subject to additional assumptions) to establish that $\sumst \sA^b \xi_{si} / (\sumst \sA^{2b}) = O(\sqrt{\log t})$. Usually, $\sqrt{\sumst \sA^{2b}}=O(\sqrt{t})$ holds. Thus, we have $\sumst \sA^b \xi_{si} \big/ \sqrt{\sumst \sA^{2b}} = O(\sqrt{t\log t})$.
		However, this bound is much looser than the concentration result in Proposition \ref{prop:concen-mono}. Because in \eqref{equ:concen-mono-result}, we prove \\ $\sumst \sA^b \xi_{si} \big/ \sqrt{\sumst \sA^{2b}} = O(\log t)$. The reason for the loose result is that their analysis is designed for the general martingale process but does not exploit the monotonic structure inherent in our problem. 
	\end{remark}}
	
	We defer the detailed proof of Proposition \ref{prop:concen-mono} to Appendix and focus on the analysis of Theorem \ref{thm:SE-indep-up} here.
	
	\textbf{Sketch of the Proof of Theorem \ref{thm:SE-indep-up}.} 
	Proposition~\ref{prop:concen-mono} provides a concentration inequality for $\muone$ for any allocation sequence $\bm{A}_i$. 
	We use the result to construct $CI(t,i)$ (defined in Algorithm \ref{alg:SE}), $UCB(t,i)=\muone+CI(t,i)$ and $LCB(t,i)=\muone-CI(t,i)$, and show that $\mu_i \in [LCB(t,i),UCB(t,i)]$ with a high probability. 
	We use it to show that (1) the optimal arm $i^*$ will most likely survive at the end of $T$ period; (2) for an arm $i$ to survive for a long time, the optimality gap $\Delta_i$ must be small, as must the regret incurred by arm $i$.
	
	To show the above two claims, we define the good event 
	\begin{equation}
		G_t=\{\mu_{i^*} \le UCB(t,i^*), \ \mu_i \ge LCB(t,i) \  \text{for} \ i \neq i^*\}.
	\end{equation}
	By the concentration result \eqref{equ:concen-mono-result}, we have $\PR(G_t^c) \le K t^{-2}$. Under $G_t$, the optimal arm always survives as $UCB(t,i^*) \ge \mu_{i^*} \ge \mu_i \ge LCB(t,i)$. 
	For an active arm $i$, we have $LCB(t,i^*) \le UCB(t,i)$, which implies $\hat{\mu}_{ti^*}-CI(t,i^*) \le \hat{\mu}_{ti}+CI(t,i)$, and
	\begin{equation*}
		\Delta_i=\mu_{i^*}-\mu_i \leq \hat{\mu}_{ti^*}-\hat{\mu}_{ti}+CI(t,i^*)+CI(t,i) \le 2(CI(t,i^*)+CI(t,i))=4 CI(t,i),
	\end{equation*}
	where the last equality is because all active arms $i$ have the same resource allocation sequence $\bm{A}_i$ and $CI(t,i)$.
	
	Let $S_{\tau_i,i}$ denote the total number of pulls of arm $i$ until period $\tau_i$ and recall the sequence $\muone$ showed in \eqref{equ:mean-one}. The regret incurred by arm $i$ is 
	\begin{align*}
		S_{\tau_i,i} \Delta_i \le 4 S_{\tau_i,i} CI(\tau_i, i)
		\le 4\cdot 2 \sqrt{3}\sigma\log T \big/\sqrt{R_{\tau_i,i}^{(1)}} \cdot S_{\tau_i,i}
		\le 8 \sqrt{3} \sigma \log T \sqrt{L_{\tau_i,i}} \ , 
	\end{align*}
	where $L_{\tau_i,i}\coloneqq S_{\tau_i,i}^2\big/R_{\tau_i,i}^{(1)}=\sum_{s=1}^t A_{si}^{2b}$.
	The total regret can be divided into two parts depending on whether the event $G_t$ occurs: 
	\begin{align*}
		R_{\pi}(T) \le \sum_{i=1}^K S_{\tau_i,i} \Delta_i + \sum_{t=1}^T \PR(G_t^c)
		=O\left(\log T \sum_{i=1}^K \sqrt{L_{\tau_i,i}} \right)+ K \sum_{t=1}^T t^{-2}.
	\end{align*}
	Note that the dominant term in $R_{\pi}(T)$ is $O\left(\sum_{i=1}^K \sqrt{L_{\tau_i,i}}\right)$ which is a function of $\bm{A}_i$ sequence and is further determined by stopping times $\bm{\tau}$. 
	The value of $\bm{\tau}$ depends on the problem instance. 
	%	For example, if $\mu_1 \gg \mu_2=\ldots=\mu_K$, then arms $2$ to $K$ will be eliminated at the first several periods and only arm $1$ survive. So $T=\tau_1 \gg \tau_2, \ldots, \tau_K$. If $\mu_1=\mu_2=\ldots=\mu_K$, then all the arms are likely to survive. So $\tau_1=\tau_2=\ldots=\tau_K=T$.
	
	We provide an upper bound of $R_{
		\pi}(T)$ for any possible value of $\bm{\tau} \in [T]^K$. The following lemma shows the maximum value of $\sum_{i=1}^K \sqrt{L_{\tau_i,i}}=O(K^{1-b} \sqrt{T})$. Thus, the total regret $R_{
		\pi}(T)$ has the upper bound $O(K^{1-b} \sqrt{T} \log T)$.

	\begin{lemma}\label{lem:gap-indep-L-bound}
		For any fixed $ \tau_1,\tau_2,\ldots,\tau_K \in \{1,2,\ldots,T\}$ and $b \in [1/2,1]$, we have 
		\begin{equation}\label{equ:gap-indep-L-up}
			\sum_{i=1}^K \sqrt{\sum_{s=1}^{\tau_i} \sA^{2b}} \le \sqrt{\frac{1}{2-2b}} K^{1-b} \sqrt{T},
		\end{equation}
		where $\sA$ is defined in \eqref{equ:inden-SE-A-tau}.
	\end{lemma}
	Note that the analysis of Lemma \ref{lem:gap-indep-L-bound} is highly non-trivial because the maximum value in \eqref{equ:gap-indep-L-up} cannot be obtained by the extreme cases of $\bm{\tau}$. 
	For example, when $\tau_1=\ldots=\tau_{K-1}=1$ and $\tau_K=T$ (all but one arms are eliminated in period one), we have $\sum_{i=1}^K \sqrt{L_{\tau_i,i}} \approx \sqrt{T}$. 
	On the other hand, when $\tau_1=\ldots=\tau_{K}=T$ (all arms remain active in period $T$), $\sum_{i=1}^K \sqrt{L_{\tau_i,i}} =K \sqrt{T K^{-2b}}=K^{1-b} \sqrt{T}$, not achieving the maximum. 
	
	We prove Lemma \ref{lem:gap-indep-L-bound} by constructing a special sequence $\{d_i\}_{i=1}^K$. Specifically, let $c_i^2=\sum_{s=1}^{\tau_i} \sA^{2b}$. We choose $d_i^2=(K-i+1)^{2b}-(k-i)^{2b}$ to make sure that $\sum_{i=1}^K c_i^2 d_i^2=\tau_K\le T$ and $\sum_{i=1}^K 1/d_i^2 \le K^{2-2b}/(2-2b)$. By the Cauthy-Schwarz inequality, we have the LHS of \eqref{equ:gap-indep-L-up} $=\sum_{i=1}^K c_i \le \sqrt{\sum_{i=1}^K c_i^2 d_i^2} \cdot \sqrt{\sum_{i=1}^K 1/d_i^2}=$ the RHS of \eqref{equ:gap-indep-L-up}.
	The optimal $\bm{\tau}$ is the solution to a system of quadratic equations, so it doesn't have a closed form. The detailed proof for Lemma \ref{lem:gap-indep-L-bound} is shown in the Appendix.
	
	% In summary, we have the worst-case regret incurred by Algorithm \ref{alg:SE}, $R(T)=O(K^{1-b} \sqrt{T})$. The detailed proof for Theorem \ref{thm:SE-indep-up} shows in Appendix.

	%\subsection{Analysis for Proposition \ref{prop:concen-mono}}\label{sec:prop-concen}
	
	%	The Proposition \ref{prop:concen-mono} shows a concentration inequality for monotone coefficients. As far as we know, such type of inequality never occurs in the literature. So the analysis for Proposition \ref{prop:concen-mono} could be of independent interest. So we show the complete proof here. 

	\section{Gap-dependent Regret Bound}\label{sec:gap-dep-up}
	In this section, we propose an algorithm that achieves the optimal rate of regret in the gap-dependent case. 
	% \nc{we need to capitalize the section title given we capitalize textbf}
	% \subsection{Gap-dependent lower bound for $b\in[0,1/2]$}
	%The definition of the gap-dependent bound is different from \eqref{equ:gap-inde-def}. 
	We first provide a lower bound for the regret.
	% Consider the $K$-armed problem instance $v$ with the mean $\{\mu_i\}_{i=1}^K$ and noises $\xi_{ti}$ in \eqref{equ:def-feedback}.
	We define the environment as
	$\PP \coloneqq \{v=(\mu_1,\ldots,\mu_K), \ \text{where} \ \mu_i \in \left[0, \infty\right) \ \text{for} \ i=\{1,\ldots,K\}\}$.
	For an instance $v \in \PP$, recall that $\Delta_i$ is the suboptimality gap of arm $i$ in $v$, i.e., $\Delta_i=\mu_{i^*} - \mu_i$. 
	The next theorem provides the gap-dependent lower bound for the regret.
	Note that we use $R_\pi(T,v)$ to highlight the dependence on the instance.
	% \nc{Is the result correct? for any $p\ge 0$ or $p=1$} \lwh{After checking the proof and MAB book, I find it's $p>0$. For $p=0$, $R(T,v) \neq 0$. For $p=1$, the condition is too loose to give any positive lower bound.}
	% \nc{sub-optimality or suboptimality, need consistent}
	% \lwh{That's a good point. I change all the ``sub-optimal'' to ``suboptimal''.}
	\begin{theorem}\label{thm:gap-dep-low}
		Suppose $\xi_{ti} \sim N(0,1)$ and $b\in [0,1/2]$. Let $\pi$ be a consistent policy over $\PP$, i.e., the regret satisfies
		$\lim_{T \rightarrow \infty} {R_\pi(T,v)}/{T^p}=0$,
		for any instance $v \in \mathcal{P}$ and any $p>0$.
		Then it holds that
		\begin{equation}\label{equ:gap-dep-low}
			\liminf_{T \rightarrow \infty} \frac{R_{\pi}(T,v)}{\log T} \geq \sum_{i:\Delta_i >0} \frac{2}{\Delta_i}, \ \ \forall v \in \mathcal{P}.
		\end{equation}
		%for any $v \in \mathcal{P}$.
	\end{theorem}
	In Theorem~\ref{thm:gap-dep-low}, we restrict the policy $\pi$ to be consistent: for any $v \in \PP$, the regret grows in a sub-polynomial order of $T$. 
	This condition is imposed to eliminate some unreasonable policies which may achieve regret much smaller than \eqref{equ:gap-dep-low} in special instances. 
	% For example, the policy that always chooses the first action achieves zero regrets for the instances when the first arm is optimal, but achieves linear-growth regret for the other instances. So this policy is unreasonable and should be eliminated.
	The result \eqref{equ:gap-dep-low} shows that the asymptotic regret lower bound of our problem is the same as SMAB for $b\le 1/2$ and the algorithms for SMAB will achieve the optimal regret for our problem.
	% The proof follows from the Le Cam's method. 
	% The key difference is that the KL-divergence depends on $B_{Ti}$ instead of $S_{Ti}$. And $B_{Ti} \le S_{Ti}$ when $b \le 1/2$.  
	We do not develop a lower bound for $b>1/2$, because as we shall see in the next section, there exists policies or algorithms that attain finite regret.
	
	\subsection{Algorithm for the Gap-dependent Bound when $b > 1/2$}
	% When $b \le 1/2$, we have shown the regret lower bound matching with the SMAB. So there's no need to design a new algorithm. In this subsection, we propose algorithms based on $\epsilon$-greedy  achieving constant regret when $b \in (0,5,1)$.
	Recall the mean estimators $\muone$ and $\mutwo$ in Section~\ref{sec:alg-ind}.
	Here, we use $\mutwo$ to design an algorithm that achieves a finite gap-dependent regret bound for $b>1/2$, because it has a benign property demonstrated by the following example:
	% for the gap-dependent bound. This is because $\mutwo$ has a much smaller variance than $\muone$ if we carefully design the allocation sequence $\bm{A}_i$.
	\begin{example}\label{exa:Rone-two}
		Suppose  $\tA=t^{-\alpha}$ for  $\alpha>1$. We have $\tS=\sumst \sA=O(1)$, $\tRone=O(1)$ and $\tRtwo=t^{1-\alpha(2-2b)}$. If $\alpha<1/(2-2b)$, then $\tRtwo \rightarrow \infty$.
	\end{example}
	In Example \ref{exa:Rone-two}, although the total resources $\tS$ allocated to arm $i$ by period $t$ is finite,
	the mean estimator $\mutwo$ can achieve increasing precision because the inverse standard deviation $\tRtwo$ tends to infinity.
	This property is not shared by $\muone$.
	% It implies that one can allocate only finite total resource to an arm ($\sA$) but the mean estimator $\mutwo$ can be increasingly precise (infinite inverse standard deviation $\tRtwo$).
	Example \ref{exa:Rone-two} sheds light on algorithm design: if we can allocate resource $\tA=t^{-\alpha}$ for arm $i$, then the regret incurred by exploration will be bounded and $\mu_i$ can be precisely estimated.
	%	\begin{remark}[Brownian Motion]
		%		When $b=1/2$, 
		%	\end{remark}
	
	%\subsection{Adaptive $\epsilon$-greedy Algorithm}
	%We first introduce the $\epsilon$-greedy algorithm and its results when applying to SMAB. Then, we apply it to RROA and compare the differences.
	
	The algorithm is based on the $\epsilon$-greedy policy for SMAB \citep{auer2002finite,seldin2014one,seldin2017improved,rouyer20a,bian2022maillard}. 
	We allocate resource $\epsilon_{ti}=K^{-1} t^{-\alpha}$ to suboptimal arms according to the estimation so far, where $\alpha$ satisfies $0 < \alpha < 1/(2-2b)$. 
	The steps are shown in Algorithm \ref{alg:eps-greedy} and the analysis is shown in Theorem~\ref{thm:gap-dep-up-greedy}.
	%\begin{algorithm}[htbp]
	%	\caption{$\epsilon$-greedy}
	%	\label{alg:eps-greedy}
	%	%\SetAlgoNoLine
	%	\KwIn{constant $b$, $\alpha=1+\epsilon/(2-2b), \epsilon \in (0,2b-1)$}
	%	\For{$t=1,2,\ldots,T$}{
		%		$\epsilon_{ti}=K^{-1} t^{-\alpha}$\\
		%		\eIf{$t=1$}{
			%			Uniformly divide the resources $A_{ti}=1/K$\\
			%		}{
			%			Pick the current optimal $i^* \in \arg\max_{i \in [K]} \mutwo$\\
			%			$
			%			A_{ti}=
			%			\left\{
			%			\begin{array}{ll}
				%			\epsilon_{ti}, &  i \neq i^*; \\
				%			1- \sum_{j \neq i^*} A_{tj}, & i=i^*;
				%			\end{array}
			%			\right.
			%			$
			%		}
		%		Allocate the resources $A_{ti}$ and observe $\{Y_{ti}\}_{i=1}^K$.
		%	}
	%\end{algorithm}
	
	\begin{algorithm}[htbp]
		\caption{$\epsilon$-greedy}
		\label{alg:eps-greedy}
		\begin{algorithmic}[1]
			\STATE \textbf{Input}: $b$, $\alpha=1+\epsilon/(2-2b), \epsilon \in (0,2b-1)$
			\FOR{$t=1,2,\ldots,T$}
			%\STATE $\epsilon_{ti}=K^{-1} t^{-\alpha}$
			\IF{$t=1$}
			\STATE Uniformly divide the resource $A_{ti}=1/K$
			\ELSE
			\STATE Pick the current optimal $i^* \in \arg\max_{i \in [K]} \mutwo$
			\STATE Allocate $
			A_{ti}=
			\left\{
			\begin{array}{ll}
				K^{-1} t^{-\alpha}, &  i \neq i^* \\
				1- \sum_{j \neq i^*} A_{tj}, & i=i^*
			\end{array}
			\right.
			$
			and observe $\{Y_{ti}\}_{i=1}^K$
			\ENDIF
			\ENDFOR
		\end{algorithmic}
	\end{algorithm}
	
	\begin{theorem}[Regret for Algorithm~\ref{alg:eps-greedy}]\label{thm:gap-dep-up-greedy}
		For $b \in \left(1/2,1\right]$, Algorithm \ref{alg:eps-greedy} achieves finite gap-dependent regret 
		\begin{equation}
			\left(\frac{2(1\!-b\!+\!\epsilon)}{K \epsilon}+2^{\lceil 1/(2b\!-\!1\!-\!\epsilon)\rceil\!+\!4}\right) \sumiK \Delta_i\!+\!8 \sumiK \left(\frac{16 \sigma^2 K^{2-2b} \lceil 1/(2b\!-\!1\!-\!\epsilon) \rceil}{\Delta_i^2}\right)^{\lceil 1/(2b\!-\!1\!-\! \epsilon)\rceil} \Delta_i. \label{equ:gap-dep-up-regret}
		\end{equation}
	\end{theorem}
	
	Note that the regret in \eqref{equ:gap-dep-up-regret} does not grow in $T$. 
	When $b=1$ and $\epsilon \approx 0$, then $\Delta_i^{1-2/\lceil 2b-1-\epsilon \rceil} \approx \Delta_i^{-1}$ and the regret is $O\left(\sumiK \Delta_i^{-1}\right)$ which is close to the regret of the full-information case, i.e., $O\left(\log K \Delta_{\min}^{-1})\right)$. 

 	\RV{We would like to point out that Algorithm \ref{alg:eps-greedy} can not achieve the optimal gap-independent bound $O(K^{1-b} \sqrt{T})$. To see this, recall the gap-independent bound
 	\begin{align*}
	 	R(T) &= \sum_{i=1}^K \Delta_i \EX[S_{Ti}] \le \sum_{i: \Delta_i < \Delta} \Delta \EX[S_{Ti}]+\sum_{i: \Delta_i \ge \Delta} \Delta_i  \EX[S_{Ti}] \le \Delta T+K \Delta^{1-2/(2b-1)} \\
	 	&= O\left(K^{b-1/2} T^{1.5-b}\right),
	 	\end{align*}
 	where the rate in $T$ is far from optimal.}
 
 	\section{Conclusion}
 	\RV{In this paper, we address an online resource allocation problem involving  divisible and renewable resources. We propose a reward framework where the mean of reward is proportional to the allocated resource and the variance is proportional to an order $b$ of the allocated resource. This framework provides a smooth connection between the standard stochastic multi-armed bandits and online learning with full feedback. We propose two algorithms that attain the optimal gap-dependent and gap-independent regret bounds.}
 	
 	\RV{The paper makes several technical contributions, utilizing novel proof techniques. First, the proof for the concentration inequality (Proposition \ref{prop:concen-mono}) represents significant breakthrough. Its fundamental nature allows for the application to other algorithms. Second, we propose new types of mean estimators ($\muone,\mutwo$) in the algorithms. These estimators and their analyses have not appeared in the bandit literature. Third, in the analysis of the algorithms, we introduce the $\tRone$ and $\tRtwo$ sequences to track the system state, in addition to the total resources $\tS$. The analysis for the sequences is a novel contribution to the literature. Lastly, our paper presents  a regret bound of $O(\sqrt{T} K^{1-b})$, where we believe it's the first result to demonstrate a power function dependence on $K$ other than $\sqrt{K}$ or $\sqrt{\log K}$. To establish  the new regret bound, we develop a new technical method based on a careful construction using the Cauthy-Schwartz inequality (Lemma \ref{lem:gap-indep-L-bound}).}
 	
 	\RV{In terms of future research directions, we highlight two possibilities. First, it would be interesting to explore whether it is possible to use a unified estimator (as mentioned in Remark \ref{remark:alter-estimator}) to develop a single algorithm that achieves optimality in both regret bounds. Second, as suggested in Remark \ref{remark:chaining}, an intriguing question is to consider the use of chaining techniques to establish the concentration result presented in Proposition \ref{prop:concen-mono}.}
	
	\newpage

% Acknowledgments---Will not appear in anonymized version
%\acks{We thank a bunch of people and funding agency.}

%\bibliographystyle{ACM-Reference-Format}
\bibliography{Renew_Alloca_final}

\setcounter{equation}{0}
\renewcommand{\theequation}{\thesection\arabic{equation}}

%\newtheorem{thm}{Theorem}[section]
%\renewcommand{\thethm}{\thesection.\arabic{thm}}
% Appendix
\newpage
\appendix
\section{Proofs in Section~\ref{sec:gap-indep-regret}}

\begin{lemma}[KL-divergence Decomposition]\label{lem:diver-decom}
	Let $v=\{P_1,\ldots,P_K\}$ be the reward distributions associated with a $K$-armed bandit where $P_i = N(\mu_i,1)$, and let $v'=\{P_1',\ldots,P_K'\}$ be another reward distributions where $P_i' = N(\mu_i',1)$. Fix some policy $\pi$ and let $\PR_v$ and $\PR_{v'}$ be the probability measures on the bandit problem induced by the $t$-round interconnection of $\pi$ and $v$ (respectively, $\pi$ and $v'$). Let $A_{ti}=(\pi_t(H_t))_i$. Then, the KL divergence satisfies
	\begin{equation*}
		KL\left(\PR_v,\PR_{v'}\right)=\sum_{i=1}^{K} \EX_v\left[\sum_{s=1}^t A_{si}^{2-2b}\right] KL(P_i,P_i'),
	\end{equation*}
	where $0 \le b<1$.
\end{lemma}

\begin{proof}
	The proof basically follows from Lemma 15.1 in \citet{lattimore2020bandit}. 
	We present the whole proof for completeness.
	
	Let $h_t=\left\{a_{11},a_{12},\ldots,a_{1K},y_{11},\ldots,y_{1K},\ldots,a_{t-1,1},\ldots,a_{t-1,K},y_{t-1,1},\ldots,y_{t-1,K}\right\}$
	be a realized sample path of $H_t$ defined in  \eqref{equ:his-def}. We have the probability density of the sample path $h_t$ under $\PR_v$ is 
	\begin{equation*}
		p_{v} (h_t)=\prod_{s=1}^t \prod_{i=1}^K (\pi_s(h_{s-1}))_i p_v(y_{si}|a_{si})=\prod_{s=1}^t \prod_{i=1}^K a_{si}  p_v(y_{si}|a_{si}),
	\end{equation*}
	where $p_v(y_{si}|a_{si})$ is the probability that $Y_{si}(a_{si})=y_{si}$ under $v$.
	The density of $\PR_{v'}$ is identical except that $p_v$ is replaced by $p_{v'}$.
	Then, we have 
	\begin{equation*}
		\frac{p_{v} (h_t)}{p_{v'} (h_t)}=\dfrac{\prod_{s=1}^t \prod_{i=1}^K a_{si}  p_v(y_{si}|a_{si})}{\prod_{s=1}^t \prod_{i=1}^K a_{si}  p_{v'}(y_{si}|a_{si})}=\dfrac{\prod_{s=1}^t \prod_{i=1}^K p_v(y_{si}|a_{si})}{\prod_{s=1}^t \prod_{i=1}^K  p_{v'}(y_{si}|a_{si})},
	\end{equation*}
	where the last equality holds because the terms involving the policy $\pi$ cancel.
	Taking logarithms of both sides, we have
	\begin{equation*}
		\log \frac{p_{v} (h_t)}{p_{v'} (h_t)}=\sum_{s=1}^t \sum_{i=1}^K \log \frac{p_v(y_{si}|a_{si})}{p_{v'}(y_{si}|a_{si})}.
	\end{equation*}
	And taking expectations of both sides with respect to $v$, we have
	\begin{equation}\label{equ:KL-path}
		\EX_v \left[\log \frac{p_{v} (H_t)}{p_{v'} (H_t)}\right]=\sum_{s=1}^t \sum_{i=1}^K \EX_v \left[\log \frac{p_v(Y_{si}|A_{si})}{p_{v'}(Y_{si}|A_{si})}\right].
	\end{equation}
	Note that
	\begin{equation}\label{equ:KL-single}
		\EX_v \left[\log \frac{p_v(Y_{si}|A_{si})}{p_{v'}(Y_{si}|A_{si})}\right]
		=\EX_v \left[ 
		\EX_v \left[\log \frac{p_v(Y_{si}|A_{si})}{p_{v'}(Y_{si}|A_{si})}\Big| A_{si}\right]\right]=\EX_v \left[KL(P_{A_{si}}, P'_{A_{si}})\right],
	\end{equation}
	where the last equation holds by the definition of KL-divergence and the last expectation is taken with respect to $A_{si}$.
	Recall that 
	\begin{equation*}
		Y_{si}(A_{si})=A_{si} \mu_i+A_{si}^{b} \xi_{si} \sim N(A_{si} \mu_i,A_{si}^{2b}).
	\end{equation*}
	We have
	\begin{equation}\label{equ:KL-normal}
		KL(P_{A_{si}}, P'_{A_{si}})=\frac{A_{si}^2 (\mu_{i}-\mu_{i}')^2}{2 A_{si}^{2b}}=A_{si}^{2-2b} (\mu_{i}-\mu_{i}')^2/2 =A_{si}^{2-2b} KL(P_i,P_i').
	\end{equation}
	Recalling the definition of KL-divergence, we have
	\begin{equation*}
		KL\left(\PR_v,\PR_{v'}\right)=\EX_v \left[\log \frac{p_{v \pi} (H_t)}{p_{v' \pi} (H_t)}\right]=\sum_{i=1}^K \EX_v \left[\sum_{s=1}^t A_{si}^{2-2b} \right] KL(P_i,P_i'),
	\end{equation*}
	where the last equation holds by \eqref{equ:KL-path}, \eqref{equ:KL-single}, \eqref{equ:KL-normal}. 
\end{proof}

\begin{proof} [of Theorem \ref{thm:gap-ind-low}]:
	The proof basically follows from Theorem 15.2 in \citet{lattimore2020bandit}. The key difference is that the total allocated resource $\tS$ is a continuous random variable. 
	
	Fix a policy $\pi$. Let $\Delta$ be a constant to be chosen later. We consider the specific bandit instance $v \in \mathcal{P}$ with mean vector $\bm{\mu}=(\Delta,0,\ldots,0)$. Let $S_{Ti} \coloneqq \sum_{t=1}^T A_{ti}$ and $i=\argmin_{j \in \{2,\ldots,K\}} \EX_v[S_{Tj}]$ where $\EX_v$ denotes the expectation under instance $v$. Since $\sum_{j=1}^K \EX_v[S_{Tj}]=T$, it holds that $\EX_v[S_{Ti}] \le T/(K-1)$. We define the second bandit instance $v' \in \mathcal{P}$ with $\mu_j'=\mu_j$ for $j \neq i$ and $\mu_i'=2 \Delta$. So the arm $i$ is the best for $v'$. 
	
	Now we choose $A=\left\{S_{T1} \le T/2\right\}$. Under event $v$, arm $1$ is optimal, the regret has the lower bound 
	\begin{equation}\label{equ:LB-indep-v}
		R_{\pi}(T,v) \ge \frac{T}{2} \PR_v(A) \Delta,
	\end{equation} 
	where $\PR_v$ is the probability measure on the sample path induced by the $T$-round interconnection of $\pi$ and $v$.
	While under event $v'$, arm $i$ is optimal, the regret is lower bounded by
	\begin{equation}\label{equ:LB-indep-v'}
		R_{\pi}(T,v') \ge \frac{T}{2} \PR_{v'}(A^c) \Delta.
	\end{equation} 
	Combining \eqref{equ:LB-indep-v} and \eqref{equ:LB-indep-v'}, we have
	\begin{equation}\label{equ:LB-ind-sumR}
		R_{\pi}(T,v)+R_{\pi}(T,v') \ge \frac{T \Delta}{2} \left(\PR_v(A)+\PR_{v'}(A^c)\right) \ge \frac{T \Delta}{4} \exp\left(-KL(\PR_v,\PR_{v'})\right),
	\end{equation}
	where the last inequality follows by Theorem 14.2 in \citet{lattimore2020bandit}.
	By Lemma \ref{lem:diver-decom}, we have
	\begin{equation}\label{equ:KL-Asi-B-in}
		KL\left(\PR_v,\PR_{v'}\right)=\sum_{j=1}^{K} \EX_v\left[\sum_{t=1}^T A_{tj}^{2-2b}\right] KL(P_j,P_j')=\EX_v\left[\sum_{t=1}^T A_{ti}^{2-2b}\right] 2 \Delta^2=2\EX_v[B_{Ti}]  \Delta^2,
	\end{equation}
	where $B_{Ti} \coloneqq \sum_{t=1}^T A_{ti}^{2-2b}$. 
	
	For $b \le 1/2$, we have $A_{ti}^{2-2b} \le A_{ti}$. So $\EX_v[B_{Ti}] \le \EX_v[S_{Ti}] \le T/(K-1)$. According to \eqref{equ:LB-ind-sumR}, \eqref{equ:KL-Asi-B-in}, we have 
	\begin{align}\label{equ:LB-ind-plus-low}
		R_{\pi}(T,v)+R_{\pi}(T,v') &\ge \frac{T \Delta}{4} \exp\left(-2\EX_v[B_{Ti}]  \Delta^2\right) \overset{(a)}{\ge} \frac{T \Delta}{4} \exp\left(-2 \Delta^2 T/(K-1)\right)\notag\\ &\overset{(b)}{\ge} \frac{\exp(-1/2)}{8} \sqrt{T(K-1)} \ge \frac{2}{27} \sqrt{T(K-1)},
	\end{align}
	where $(a)$ follows from $B_{Ti} \le T/(K-1)$, $(b)$ follows from choosing $\Delta=\sqrt{(K-1)/4T}$.
	
	For $b > 1/2$, we have $A_{ti}^{2-2b} \ge A_{ti}$ and $B_{Ti} \ge S_{Ti}$. Moreover, we have
	\begin{equation*}
		\frac{1}{T} B_{Ti}=\frac{1}{T} \sum_{t=1}^T A_{ti}^{2-2b} \overset{(a)}{\le} \left(\frac{1}{T}\sum_{t=1}^T A_{ti}\right)^{2-2b}=\left(\frac{S_T}{T}\right)^{2-2b},
	\end{equation*}
	where $(a)$ follows from the concavity of the function $f(x)=x^{2-2b}$ and Jensen's inequality.
	So
	\begin{equation}\label{equ:LB-ind-BTi-up}
		\EX_v[B_{Ti}] \le T^{2b-1} \EX_v[S_{Ti}^{2-2b}] \le  T^{2b-1} \left(\frac{T}{K-1}\right)^{2-2b}=\frac{T}{(K-1)^{2-2b}}.
	\end{equation}
	Combining \eqref{equ:LB-ind-sumR}, \eqref{equ:KL-Asi-B-in}, \eqref{equ:LB-ind-BTi-up}, we have 
	\begin{align} \label{equ:LB-ind-plus-lowb}
		R_{\pi}(T,v)+R_{\pi}(T,v') &\ge \frac{T \Delta}{4} \exp\left(-2\EX_v[B_{Ti}]  \Delta^2\right) \overset{(a)}{\ge} \frac{T \Delta}{4} \exp\left(-2 \Delta^2 T/(K-1)^{2-2b}\right) \notag\\ &\overset{(b)}{\ge} \frac{\exp(-1/2)}{8} (K-1)^{1-b}\sqrt{T} \ge \frac{2}{27} (K-1)^{1-b}\sqrt{T},
	\end{align}
	where $(a)$ holds by \eqref{equ:LB-ind-BTi-up}, $(b)$ holds by choosing $\Delta=(K-1)^{1-b} /\sqrt{4T}$.
	Since the bandit instance can be chosen from $v$ or $v'$, the minimax regret
	\begin{align*}
		R_{\pi}(T) 
		&\ge \max\{R_{\pi}(T,v),R_{\pi}(T,v')\} \ge \frac{1}{2} (R_{\pi}(T,v)+R_{\pi}(T,v')) \\
		&\ge \left\{
		\begin{array}{lc}
			\frac{1}{27} \sqrt{(K-1)T}, &  \text{for} \ 0 \leq b \le 1/2, \\
			\frac{1}{27} (K-1)^{1-b}\sqrt{T} &  \text{for} \ 1/2 < b \le 1,
		\end{array}
		\right.
	\end{align*}
	where the last inequality holds by summarizing \eqref{equ:LB-ind-plus-low}, \eqref{equ:LB-ind-plus-lowb}.
\end{proof}

\begin{proof}[of Proposition \ref{prop:concen-mono}]:
	We prove it in the following three steps.
	
	\textbf{Step one:} to utilize the monotonic property of sequence $\bm{a}$, we transform random variables $\{\xi_s\}_{s=1}^t$ to $\{f_s\}_{s=1}^t$ and transform the sequence $\bm{a}$ to a positive sequence $\bm{b}$, where $\{f_s\}_{s=1}^t$ and  $\bm{b}$ are defined as 
	\begin{equation*}
		\left\{
		\begin{array}{ll}
			f_1  \coloneqq \xi_1+\xi_2+\ldots+\xi_t\\
			f_2 \coloneqq \xi_2+\ldots+\xi_t\\
			\quad \vdots\\
			f_t \coloneqq \xi_t
		\end{array}
		\right. , \quad \quad
		\left\{
		\begin{array}{ll}
			b_1  \coloneqq a_1\\
			b_2 \coloneqq a_2-a_1\\
			\quad \vdots\\
			b_t \coloneqq a_t-a_{t-1}
		\end{array}
		\right..
	\end{equation*}
	Thus, the numerator in \eqref{equ:concen-mono-result} is 
	\begin{align*}
		\sum_{s=1}^t a_{s} \xi_{s}= a_1 \sumst \xi_s+ (a_2-a_1) \sum_{s=2}^t \xi_s+ \ldots + (a_t-a_{t-1}) \xi_t=\sumst b_s f_s.
	\end{align*}
	Figure~\ref{fig:sequence_add} illustrates the transformation. 
	The summation $\sum_{s=1}^t a_{s} \xi_{s}$ is represented by the volume of the shadow area, which can be counted in the columns (adding up in $\bm{a}$) and the rows (adding up in $\bm{b}$). %The heights of the columns shown in $\bm{a}$ and the lengths of the rows shown in $\bm{b}$.
	\begin{figure}
		\centering
		\includegraphics[width=0.7\textwidth]{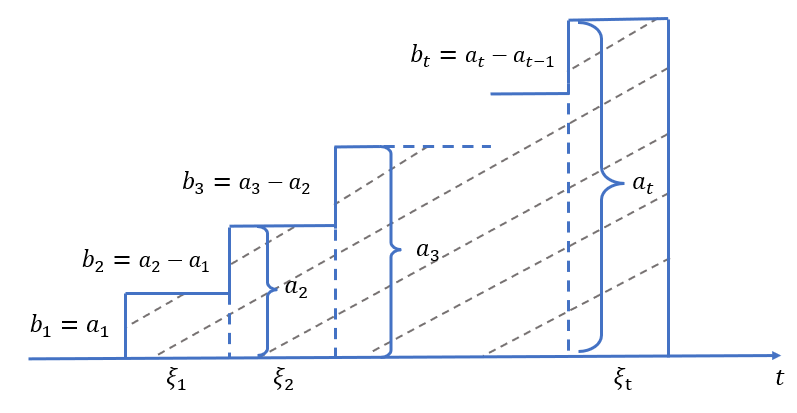}
		\caption{Summation of sequence $\xi_s$ and $f_s$}
		\label{fig:sequence_add}
	\end{figure}
	The denominator in \eqref{equ:concen-mono-result} is $\sqrt{\sumst a_s^2}=\sqrt{\sumst \left(\sum_{i=1}^s b_i\right)^2}$.
	Since $\bm{a} \in \mathcal{A}_m$,  we have $b_1=a_1>0$, $b_s=a_s-a_{s-1} \ge 0$ for $s=\{2,3,\ldots,t\}$. Thus, let $\mathcal{B} \coloneqq \{b_1>0, b_2 \ge 0, \ldots b_t \ge 0\}$, we have 
	\begin{equation}\label{equ:gap-indep-sup-ab}
		\sup_{\bm{a} \in \mathcal{A}_m}\frac{\sum_{s=1}^t a_{s} \xi_{s}}{\sqrt{\sum_{s=1}^t a_{s}^2}} = \sup_{\bm{b} \in \mathcal{B}} \dfrac{\sum_{s=1}^t b_s f_s}{\sqrt{\sum_{s=1}^t \left(\sum_{i=1}^s b_i\right)^2}}.
	\end{equation}
	
	\textbf{Step two:} we give a high-probability upper bound $\epsilon_s$ for $f_s$, where
	\begin{equation}\label{equ:gap-indep-epsilon-s}
		\epsilon_s \coloneqq \sqrt{t-s+1} \epsilon,
	\end{equation}
	and $\epsilon$ is a positive constant.
	If this holds, then we can easily show that With a high probability, $\sumst b_s f_s \le \sumst b_s \epsilon_s$ because $\bm{b} \succeq 0$.
	
	Since $f_s=\sum_{i=s}^t \xi_i$ is $\sqrt{t-s+1} \sigma$-sub-Gaussian, we have 
	\begin{equation}\label{equ:gap-indep-f-epsilon}
		\PR(f_s \ge \epsilon_s) \le \exp\left(-\frac{\epsilon_s^2}{2(t-s+1) \sigma^2}\right)= \exp\left(-\frac{\epsilon^2}{2 \sigma^2}\right).
	\end{equation}
	Then, taking the union of events $\{f_s \ge \epsilon_s\}$, we have
	\begin{equation*}
		\PR \left(\cup_{s=1}^t \{f_s \ge \epsilon_s\}\right) \le \sumst \PR(f_s \ge \epsilon_s) \le t \exp\left(-\frac{\epsilon^2}{2 \sigma^2}\right).
	\end{equation*}
	We define the event 
	\begin{equation}
		G=\{\cap_{s=1}^t \{f_s \le \epsilon_s\}\},
	\end{equation}
	and obtain that
	\begin{equation}\label{equ:gap-indep-Gc}
		\PR(G^c) \le t\exp\left(-\frac{\epsilon^2}{2 \sigma^2}\right).
	\end{equation}
	Note that under the event $G$, we always have $f_s \le \epsilon_s$ for all $s$. 
	For any realization of $f_1,f_2,\ldots, f_t$, we have the corresponding vector $\bm{b}^*(\bm{f}) \in \mathcal{B} $ which maximizes  \eqref{equ:gap-indep-sup-ab}. 
	%For any vector $\bm{b} \succeq 0$, we have $\sum_{s=1}^t b_s f_s \le \sum_{s=1}^t b_s \epsilon_s$ with probablity one. 
	Thus, under event G, we have
	\begin{equation}\label{equ:gap-indep-bepsilon}
		\sup_{\bm{b} \in \mathcal{B}} \dfrac{\sum_{s=1}^t b_s f_s}{\sqrt{\sum_{s=1}^t \left(\sum_{i=1}^s b_i\right)^2}} 
		\!=\! \dfrac{\sum_{s=1}^t b_s^* f_s}{\sqrt{\sum_{s=1}^t \left(\sum_{i=1}^s b_i^*\right)^2}} 
		\!\overset{(a)}{\le}\! \dfrac{\sum_{s=1}^t b_s^* \epsilon_s}{\sqrt{\sum_{s=1}^t \left(\sum_{i=1}^s b_i^*\right)^2}} 
		\!\le\! \sup_{\bm{b} \in \mathcal{B}} \dfrac{\sum_{s=1}^t b_s \epsilon_s}{\sqrt{\sum_{s=1}^t \left(\sum_{i=1}^s b_i\right)^2}},
	\end{equation}
	where $(a)$ follows by $\bm{b} \succeq 0$. So we have
	\begin{align}\label{equ:gap-indep-bs-bepsilon}
		&{}\PR\left(\Ind \left(\sup_{\bm{b} \in \mathcal{B}} \dfrac{\sum_{s=1}^t b_s f_s}{\sqrt{\sum_{s=1}^t \left(\sum_{i=1}^s b_i\right)^2}} \ge \sqrt{\frac{3}{2} \log t}\epsilon\right) \cap G\right)  \notag\\
		&\le \PR\left(\Ind \left(\sup_{\bm{b} \in \mathcal{B}} \dfrac{\sum_{s=1}^t b_s \epsilon_s}{\sqrt{\sum_{s=1}^t \left(\sum_{i=1}^s b_i\right)^2}} \ge \sqrt{\frac{3}{2} \log t}\right)\epsilon \cap G\right).
	\end{align}
	
	\textbf{Step three:} Note that there's no randomness in \eqref{equ:gap-indep-bepsilon}. We claim that
	\begin{equation}\label{equ:gap-indep-sup-b}
		\sup_{\bm{b} \in \mathcal{B}} \dfrac{\sum_{s=1}^t b_s \epsilon_s}{\sqrt{\sum_{s=1}^t \left(\sum_{i=1}^s b_i\right)^2}}=\sup_{\bm{b} \in \mathcal{B}} \dfrac{\sum_{s=1}^t b_s \sqrt{t-s+1} \epsilon}{\sqrt{\sum_{s=1}^t \left(\sum_{i=1}^s b_i\right)^2}} < \sqrt{\frac{3}{2} \log t} \epsilon.
	\end{equation}
	If so, then $\eqref{equ:gap-indep-bs-bepsilon}=0$.
	In the following, we will prove a stronger version of \eqref{equ:gap-indep-sup-b}: for any $\bm{b} \in \mathbb{R}^t$, 
	\begin{equation}\label{equ:gap-indep-b-algebra}
		\left(\sum_{s=1}^t b_s \sqrt{t-s+1}\right)^2 < \frac{3}{2} \log(t) \sum_{s=1}^t \left(\sum_{i=1}^s b_i\right)^2. 
	\end{equation}
	It's a purely algebraic problem to prove. We write \eqref{equ:gap-indep-b-algebra} in the matrix form:
	\begin{equation}\label{equ:gap-indep-b-matrix}
		\bm{b}^T \left( A-\frac{3}{2} \log (t) B\right) \bm{b}	
		< 0 \ \Leftrightarrow \ A \prec \frac{3}{2} \log(t) B,
	\end{equation}
	where
	\begin{align}
		&A=\begin{pmatrix}
			t & \sqrt{t(t-1)} & \sqrt{t(t-2)} & \ldots & \sqrt{t}\\
			\sqrt{t(t-1)} & t-1 & \sqrt{(t-1)(t-2)} & \ldots & \sqrt{t-1}\\
			\sqrt{t(t-2)} & \sqrt{(t-1)(t-2)} & t-2 & \ldots & \sqrt{t-2}\\
			\vdots & \vdots & \vdots & \ldots & \vdots\\
			\sqrt{t} & \sqrt{t-1} & \sqrt{t-2} & \ldots & 1
		\end{pmatrix}, \\
		&B=\begin{pmatrix}
			t & t-1 & t-2 & \ldots & 1\\
			t-1 & t-1 & t-2 & \ldots & 1\\
			t-2 & t-2 & t-2 & \ldots & 1\\
			\vdots & \vdots & \vdots & \ldots & \vdots\\
			1 & 1 & 1 & \ldots & 1
		\end{pmatrix}.
	\end{align}
	Note that $A$ and $B$ have the same diagonal elements. But the off-diagonal elements of $A$ are all greater than $B$. So $A-B$ is not negative definite. We will show $A-\frac{3}{2} \log (t) B$ is negative definite.
	Note that $A$ is a rank-$1$ matrix:
	\begin{equation}\label{equ:indenp-matrix-A}
		A=\lambda_A V V^\top, \ \text{where} \  \lambda_A=\frac{t(t+1)}{2}, \  V=\sqrt{\frac{2}{t(t+1)}} 
		\begin{pmatrix}
			\sqrt{t}&
			\sqrt{t-1}&
			\cdots&
			1
		\end{pmatrix}^\top.
	\end{equation} 
	Note that $B$ is full-rank and positive definite:
	\begin{equation}\label{equ:indenp-matrix-D}
		B=D D^T, \ \text{where} \  D=D^T=\begin{pmatrix}
			1 & 1 & \ldots & 1 & 1\\
			1 & 1 & \ldots & 1 & 0\\
			\vdots & \vdots & \ldots & \vdots & \vdots\\
			1 & 1 & \vdots & 0 & 0\\
			1 & 0 & \vdots & 0 & 0
		\end{pmatrix}, \ 
		D^{-1}=\begin{pmatrix}
			0 & 0 & \ldots & 0 & 1\\
			0 & 0 & \ldots & 1 & -1\\
			0 & 0 & \ldots & -1 & 0 \\
			\vdots & \vdots &\vdots & \vdots& \vdots\\
			0 & 1  & \vdots & 0 & 0\\
			1 & -1 & \ldots  & 0 & 0
		\end{pmatrix}.
	\end{equation}
	We construct the following matrix
	\begin{equation*}
		E\coloneqq
		\begin{pmatrix}
			\lambda_A^{-1} & V^T\\
			V & \frac{3}{2} \log (t) B
		\end{pmatrix}.
	\end{equation*}
	By the property of Schur complement (Theorem 1.12 in \cite{zhang2006schur}), we have 
	\begin{equation}\label{equ:gap-indep-Schur}
		\frac{3}{2} \log (t) B \succ A =\lambda_A V V^T \iff E \succ 0 \iff \lambda_A^{-1} > V^T \left(\frac{3}{2} \log (t) B\right)^{-1} V.
	\end{equation}
	It remains to prove
	\begin{equation}\label{equ:gap-indep-Schur-inq}
		\eqref{equ:gap-indep-Schur} \iff \frac{3}{2} \log (t) > \lambda_A V^\top B^{-1} V=\lambda_A (D^{-1}V)^\top D^{-1} V= (\sqrt{\lambda_A} D^{-1}V)^\top \sqrt{\lambda_A} D^{-1} V.
	\end{equation}
	By \eqref{equ:indenp-matrix-A}, \eqref{equ:indenp-matrix-D}, we have
	\begin{equation*}
		\sqrt{\lambda_A} D^{-1} V=\begin{pmatrix}
			0 & 0 & \ldots & 0 & 1\\
			0 & 0 & \ldots & 1 & -1\\
			0 & 0 & \ldots & -1 & 0 \\
			\vdots & \vdots &\vdots & \vdots& \vdots\\
			0 & 1  & \vdots & 0 & 0\\
			1 & -1 & \ldots  & 0 & 0
		\end{pmatrix} 
		\begin{pmatrix}
			\sqrt{t}\\
			\sqrt{t-1}\\
			\vdots\\
			\sqrt{3}\\
			\sqrt{2}\\
			1
		\end{pmatrix} =
		\begin{pmatrix}
			1\\
			\sqrt{2}-1\\
			\vdots\\
			\sqrt{t-2}-\sqrt{t-3}\\
			\sqrt{t-1}-\sqrt{t-2}\\
			\sqrt{t}-\sqrt{t-1}\\
		\end{pmatrix}.
	\end{equation*}
	Thus, we have
	\begin{align*}
		(\sqrt{\lambda_A} D^{-1}V)^\top \sqrt{\lambda_A} D^{-1} V
		&=\sumst \left(\sqrt{s}-\sqrt{s-1}\right)^2
		=1+\sum_{s=2}^t \left(\dfrac{1}{\sqrt{s}+\sqrt{s-1}}\right)^2\\
		& \overset{(a)}{<} 1+\sum_{s=2}^t \left(\dfrac{1}{2\sqrt{s-1}}\right)^2
		=1+\frac{1}{4} \sum_{s=2}^t\frac{1}{s-1}
		=1+\frac{1}{4} \sum_{i=1}^{t-1}\frac{1}{i}\\
		& \overset{(b)}{\le} \frac{5}{4}+ \frac{1}{4} \int_{1}^{t-1} \frac{1}{x} dx
		\le \frac{5}{4}+ \frac{1}{4} \log t\\
		& \overset{(c)}{\le} \frac{3}{2} \log t,
	\end{align*}
	where $(a)$ follows by $\sqrt{s} \ge \sqrt{s-1}$, $(b)$ follows by $1/x$ decreasing in $x$, $(c)$ follows by $\log t \ge 1$ for $t \ge 2$.
	Thus, we have proved \eqref{equ:gap-indep-Schur-inq}, \eqref{equ:gap-indep-Schur}, \eqref{equ:gap-indep-b-matrix}, \eqref{equ:gap-indep-b-algebra} and \eqref{equ:gap-indep-sup-b}.
	
	In summary, we have
	\begin{align}
		{}&\PR\left( \sup_{\bm{a} \in \mathcal{A}_m}\frac{\sum_{s=1}^t a_{s} \xi_{s}}{\sqrt{\sum_{s=1}^t a_{s}^2}} \ge  \sqrt{\frac{3}{2} \log t}\epsilon \right) \notag\\
		&\overset{(a)}{=} \PR\left( \sup_{\bm{b} \in \mathcal{B}} \dfrac{\sum_{s=1}^t b_s f_s}{\sqrt{\sum_{s=1}^t \left(\sum_{i=1}^s b_i\right)^2}} \ge \sqrt{\frac{3}{2} \log t} \epsilon \right) \notag\\
		& \overset{(b)}{\le} \PR\left(\Ind \left(\sup_{\bm{b} \in \mathcal{B}} \dfrac{\sum_{s=1}^t b_s \epsilon_s}{\sqrt{\sum_{s=1}^t \left(\sum_{i=1}^s b_i\right)^2}} \ge \sqrt{\frac{3}{2} \log t}\epsilon\right) \cap G\right)+ \PR(G^c) \notag\\
		& \overset{(c)}{=} \PR(G^c) \overset{(d)}{\le} t \exp\left(-\frac{\epsilon^2}{2 \sigma^2}\right), \label{equ:gap-indep-sup-concen}
	\end{align}
	where $(a)$ holds by \eqref{equ:gap-indep-sup-ab}, $(b)$ holds by \eqref{equ:gap-indep-bs-bepsilon}, $(c)$ holds by \eqref{equ:gap-indep-sup-b}, $(d)$ holds by \eqref{equ:gap-indep-Gc}. Replacing $\epsilon$ by $\sqrt{\log t} \epsilon$, we have \eqref{equ:concen-mono-result}. Thus, we complete the proof of Proposition \ref{prop:concen-mono}.
\end{proof}

\begin{proof} [of Theorem \ref{thm:SE-indep-up}]:
	We use the concentration inequality for $\muone$. By \eqref{equ:gap-ind-sup-nor} and \eqref{equ:gap-indep-sup-concen}, we have 
	\begin{equation*}
		\PR\left(\bigg|(\muone-\mu_i) \sqrt{\tRone}\bigg| \ge  \sqrt{\frac{3}{2} \log t}\epsilon \right) \le t \exp\left(-\frac{\epsilon^2}{2 \sigma^2}\right)
	\end{equation*}
	for any arm $i \in [K]$, any period $t \in [T]$ and any possible value of $\muone$, $\tRone$ generated by the Algorithm \ref{alg:SE}. Let $\delta=T^{-4}$ and $\epsilon=\sqrt{2 \sigma^2 \log(1/\delta)}$.
	We have
	\begin{equation}\label{equ:gap-indep-con}
		\PR\left(\Big|(\muone-\mu_i) \sqrt{\tRone}\Big| \geq \sigma \sqrt{3  \log t \log (1/\delta)}\right) \leq t\delta.
	\end{equation}
	%According to the stopping times $\bm{\tau}$, we define the random variable $E_{ti}$ to indicate whether arm $i$ is active at period $t$:
	%\begin{equation}
	%E_{ti}=\left\{
	%\begin{array}{lc}
	%1, & \ \text{if} \ t \le \tau_i,\\
	%0, & \ \text{if} \ t > \tau_i.
	%\end{array}
	%\right.
	%\end{equation}
	%The basic idea to prove the regret bound is that (1) the optimal arm is unlikely to be eliminated in the whole periods. (2) For those arms surviving for a long time, their suboptimality gap $\Delta_i$ will be small. To show this, we define the confidence interval and good events in which $\hat{\mu}_{ti^*}^{(2)}$ is not much lower than $\mu_{ti^*}$; for a suboptimal arm $i$, $\muone$ is not much higher than $\mu_{ti}$. 
	Next, let $D_t$ to denote the active set at the end of period $t$. We define the confidence interval and the good events,
	\begin{align}
		&CI(t,i)=\sigma \sqrt{3  \log t \log (1/\delta)} \bigg/ \sqrt{\tRone}, \ \forall \ i \in [K] \label{equ:gap-indep-CI}\\
		&G_{ti^*}= \left\{\mu_{i^*} \le UCB(t,i^*) \coloneqq \hat{\mu}_{ti^*}^{(1)}+CI(t,i^*)\right\}, \label{equ:gap-indep-Gti*}\\
		&G_{ti}=\left\{\mu_i \ge LCB(t,i) \coloneqq \muone-CI(t,i)\right\}, \ \text{for} \ i \in D_t/i^*,\label{equ:gap-indep-Gti}\\
		&G_t=\left\{\cap_{i \in D_t} G_{ti} \right\}, \ G=\left\{\cap_{t=1}^T G_t\right\}. \label{equ:gap-indep-Gt}
	\end{align}
	Under event $G_{ti^*}$, $\mu_{i^*}$ is not much lower than $\hat{\mu}_{ti^*}^{(1)}$. So the optimal arm will not be eliminated. Under event $G_{ti}$,  $\mu_{ti}$ is not much lower than $\muone$. So for those arms surviving for a long time, their suboptimality gap $\Delta_i$ will be small.
	
	From \eqref{equ:gap-indep-con}, we know that 
	\begin{equation}\label{equ:gap-indep-Gc-small}
		\PR\left(G_{ti^*}^c\right) \le t\delta,\ \PR\left(G_{ti}^c\right) \le t\delta,\ 
		\PR\left(G_{t}^c\right) \le \sum_{i \in D_t}
		\PR(G_{ti}^c) \le K t\delta, \ \PR(G^c) \le \sum_{t=1}^T \PR\left(G_{t}^c\right) \le KTt \delta.
	\end{equation}
	
	According to the definition of regret \eqref{equ:def-regret}, we have
	\begin{align}
		R_{\pi}(T)
		&=\EX\left[\sum_{t=1}^T \sum_{i=1}^K  \tA (\mu_{i^*}-\mu_i)\right]=\EX\left[\sum_{t=1}^T \sum_{i=1}^K  \tA \Delta_i\right] \notag\\
		&\overset{(a)}{=}\EX\left[\sum_{t=1}^T \sum_{i=1}^K  \tA \Delta_i \Ind \left(\tau_i \ge t \right)\right] \notag\\
		&=\EX\left[\sum_{t=1}^T \sum_{i=1}^K  \tA \Delta_i \Ind \left((\tau_i \ge t) \cap G^c\right)\right]+\EX\left[\sum_{t=1}^T \sum_{i=1}^K  \tA \Delta_i \Ind \left((\tau_i \ge t) \cap G\right)\right], \label{equ:gap-indep-reget-decom}
	\end{align}
	where $(a)$ holds by $\tA=0$ if $\tau_i <t$.
	
	For the first term in \eqref{equ:gap-indep-reget-decom}, we have
	\begin{align}
		\EX\left[\sum_{t=1}^T \sumiK  \tA \Delta_i \Ind \left((\tau_i \ge t) \cap G^c\right)\right]
		&\le \EX\left[\sum_{t=1}^T \sum_{i=1}^K  \tA \Delta_{max} \Ind \left((\tau_i \ge t) \cap G^c\right)\right] \notag\\
		&\overset{(a)}{=}\Delta_{max} \EX\left[\sum_{t=1}^T \sum_{i \in D_t} \tA \Ind \left( G^c\right) \right] \notag\\
		&\overset{(b)}{=}\Delta_{max} \sum_{t=1}^T \EX\left[\EX\left[\sum_{i \in D_t} \tA \Big| G^c\right] \Ind(G^c)\right] \notag\\
		&\overset{(c)}{=}\Delta_{max} \sum_{t=1}^T \PR(G^c) \notag\\
		& \overset{(d)}{\leq} \Delta_{max} T^3 K \delta, \label{equ:gap-indep-regret-Gc}
	\end{align}
	where $\Delta_{max}=\max_i \{\Delta_i\}$, 
	$(a)$ holds by $i \in D_t \Leftrightarrow \tau_i \ge t$, $(b)$ holds by the tower rule of expectation, $(c)$ holds by $\sum_{i \in D_t} \tA=1$, $(d)$ holds by \eqref{equ:gap-indep-Gc-small}.
	
	For the second term in \eqref{equ:gap-indep-reget-decom}, we have
	\begin{align}\label{equ:gap-indep-AG}
		\EX\left[\sum_{t=1}^T \sum_{i=1}^K  \tA \Delta_i \Ind \left((\tau_i \ge t) \cap G\right)\right]
		&=\EX\left[ \sum_{i=1}^K \sum_{t=1}^T \Delta_i \tA  \Ind \left((t \le \tau_i) \cap G\right)\right] \notag\\
		&=\EX\left[\left(\sum_{i=1}^K  \sum_{t=1}^{\tau_i} \Delta_i\tA\right)  \Ind (G) \right],
		%&\overset{(b)}{\le} \EX\left[\sum_{i=1}^K \Delta_i S_{\tau_i,i}  \Ind (G_t) \right],
		%& = O(K^{1-b} \sqrt{T}  \sqrt{\log T \log(1/\delta)}).
	\end{align}
	where the last equality holds because the argument of event $G$ is about the entire sample path.
	
	By \eqref{equ:gap-indep-Gti*}, \eqref{equ:gap-indep-Gti}, \eqref{equ:gap-indep-Gt}, we have under $G_t$,
	\begin{equation*}
		UCB(t,i^*) \ge \mu_{i^*} \ge \mu_i \ge LCB(t,i),
	\end{equation*}
	for the optimal arm $i^*$ and suboptimal arm $i \in D_t/i^*$. So the optimal arm $i^*$ will not be eliminated, i.e., $i^* \in D_t$. For those arms $i \in D_t/\{i^*\}$, since they are active, we have $UCB(t,i) \ge LCB(t,i^*)$. That implies
	\begin{equation}\label{equ:gap-indep-ULi}
		\muone+CI(t,i) \ge \hat{\mu}^{(1)}_{ti^*}-CI(t,i^*).
	\end{equation}
	Then, we give an upper bound for the suboptimality gap:
	\begin{equation}\label{equ:gap-indep-Delta-up}
		\Delta_i=\mu_{i^*}-\mu_i \overset{(a)}{\leq} \hat{\mu}^{(1)}_{ti^*}-\muone+CI(t,i^*)+CI(t,i) \overset{(b)}{\le} 2(CI(t,i^*)+CI(t,i)) \overset{(c)}{=} 4 CI(t,i),
	\end{equation}
	where $(a)$ follows from \eqref{equ:gap-indep-Gti*}, \eqref{equ:gap-indep-Gti}, $(b)$ follows from \eqref{equ:gap-indep-ULi}, $(c)$ follows from that if arm $i$ and $i^*$ are active at period $t$, their allocation sequences $\bm{A}_i$ and $\bm{A}_{i'}$ should be equal, so do sequences $\bm{R}^{(2)}_i$ and $\bm{R}^{(2)}_{i'}$. 
	
	Thus, when $G_t$ happens, by \eqref{equ:gap-indep-Delta-up}, we have $\Delta_i \le 4CI(t,i)$. So  the event $G_t$ is a subset of the event $\Delta_i \le 4CI(t,i)$. Thus,
	$\Ind(G_t) \le \Ind(\Delta_i \le 4CI(t,i))$ for any $i \in [K]$. Furthermore, when $G$ happens, we have $\Delta_i \le 4CI(t,i)$ for any $i \in [K]$ and any $t \in [T]$ and $\Ind (G) \le \Ind\left(\Delta_i \le 4CI(t,i) \ \forall i, \ \forall t\right)$.
	So we have
	\begin{align}
		\eqref{equ:gap-indep-AG}
		&\le \EX\left[\left(\sum_{i=1}^K  \sum_{t=1}^{\tau_i} \Delta_i\tA\right)  \Ind (\Delta_i \le 4CI(t,i) \ \forall i, \ \forall t) \right] \notag\\
		&\overset{(a)}{\le} \EX\left[\sum_{i=1}^K  \sum_{t=1}^{\tau_i} \Delta_i\tA  \Ind (\Delta_i \le 4CI(\tau_i,i)) \right] \notag\\
		&\overset{(b)}{\le} 4\EX\left[\sum_{i=1}^K  \sum_{t=1}^{\tau_i} \tA  CI(\tau_i,i) \right]\notag\\
		&\overset{(c)}{=}4\EX\left[\sum_{i=1}^K  S_{\tau_i,i}  CI(\tau_i,i) \right], \label{equ:gap-indep-S-tau}
	\end{align}
	where $(a)$ follows from picking up $t=\tau_i$ and omit the indicator $\Ind (\Delta_i \le 4CI(t,i))$ for other $t$, $(b)$ follows from replacing $\Delta_i$ by $4CI(\tau_i,i)$, $(c)$ follows from $S_{\tau_i,i}=\sum_{t=1}^{\tau_i} \tA $.
	
	By \eqref{equ:gap-indep-CI}, we have
	\begin{equation}\label{equ:gap-indep-CI-up}
		CI(\tau_i,i)=\sigma \sqrt{3  \log \tau_i \log (1/\delta)} \bigg/ \sqrt{R^{(2)}_{\tau_i,i}} \le \sigma \sqrt{3  \log T \log (1/\delta)} \sqrt{\sum_{s=1}^{\tau_i} \sA^{2b}} \big/ S_{\tau_i,i},
	\end{equation}
	where the second equation follows by $\tau_i \le T$ and the definition of $\tRone$ in \eqref{equ:def-R2}. Thus, plugging \eqref{equ:gap-indep-CI-up} into \eqref{equ:gap-indep-S-tau}, we have
	\begin{align}\label{equ:gap-indep-Asi}
		\eqref{equ:gap-indep-S-tau} \le  4 \sigma \sqrt{3 \log T \log(1/\delta)} \sumiK \sqrt{\sum_{s=1}^{\tau_i} \sA^{2b}}.
	\end{align}
	Note that the value of \eqref{equ:gap-indep-Asi} depends on the random stopping times $\bm{\tau}$. Lemma \ref{lem:gap-indep-L-bound} gives an upper bound for \eqref{equ:gap-indep-Asi} for any possible stopping times $\bm{\tau}$.
	
	According to Lemma \ref{lem:gap-indep-L-bound}, we have 
	\begin{equation}\label{equ:gap-indep-regret-G}
		\eqref{equ:gap-indep-Asi} \le 2 \sigma \sqrt{6 \log T \log(1/\delta)/(1-b)} K^{1-b} \sqrt{T}.
	\end{equation}
	
	In summary, by \eqref{equ:gap-indep-reget-decom}, \eqref{equ:gap-indep-regret-Gc}, \eqref{equ:gap-indep-regret-G}, the total regret can be upper bounded by 
	\begin{align*}
		R_{\pi}(T) 
		&\le \Delta_{max} T^3 K \delta+ 2 \sigma \sqrt{6 \log T \log(1/\delta)/(1-b)} K^{1-b} \sqrt{T}\\
		&= \Delta_{max}+4 \sqrt{6/(1-b)} \sigma (\log T) K^{1-b} \sqrt{T},
	\end{align*}
	where setting $\delta=1/T^4$.
\end{proof}

\begin{proof}[of Lemma \ref{lem:gap-indep-L-bound}]
	Without loss of generality, we let $1 \le \tau_1 \le \tau_2 \le \cdots \le \tau_K$. By the definition of $\sA$ in \eqref{equ:inden-SE-A-tau}, we have
	\begin{align}
		{}&\sum_{i=1}^K \sqrt{\sum_{s=1}^{\tau_i} \sA^{2b}} \notag\\
		=&\sqrt{\tau_1 K^{-2b}}+\sqrt{\tau_1 K^{-2b}+(\tau_2-\tau_1) (K-1)^{-2b}}\\
		+& \sqrt{\tau_1 K^{-2b}+(\tau_2-\tau_1) (K-1)^{-2b}+(\tau_3-\tau_2)(K-2)^{-2b}}+\ldots \notag\\
		+&\sqrt{\tau_1 K^{-2b}+(\tau_2-\tau_1) (K-1)^{-2b}+(\tau_3-\tau_2)(K-2)^{-2b}+\ldots+(\tau_K-\tau_{K-1})} \notag\\
		=&\sum_{i=1}^K \sqrt{\sum_{j=1}^i (\tau_i-\tau_{i-1}) (K-j+1)^{-2b}}=\sumiK c_i. \label{equ:inden-L-bound}
	\end{align}
	Let $c_i \coloneqq \sqrt{\sum_{j=1}^i (\tau_i-\tau_{i-1}) (K-j+1)^{-2b}}$, $d_i^2=(K-i+1)^{2b}-(K-i)^{2b}$.
	By Cauchy-Schwarz inequality, we have 
	\begin{align}\label{equ:gap-indep-CS}
		\left(\sumiK c_i\right)^2
		&=\left(c_1 d_1 \frac{1}{d_1}+c_2 d_2 \frac{1}{d_2}+\ldots+c_K d_K \frac{1}{d_K}\right)^2 \notag\\
		&\le \left(c_1^2 d_1^2+c_2^2 d_2^2+\ldots+c_K^2 d_K^2\right) \left(\frac{1}{d_1^2}+\ldots+\frac{1}{d_K^2}\right).
	\end{align}
	It can be checked that 
	\begin{align}
		&c_1^2 d_1^2+c_2^2 d_2^2+\ldots+c_K^2 d_K^2 \notag\\
		&=\tau_1 K^{-2b} K^{2b}+(\tau_2-\tau_1) (K-1)^{-2b} (K-1)^{2b} +(\tau_3-\tau_2)(K-2)^{-2b} (K-2)^{2b}+\ldots \notag\\
		&+(\tau_K-\tau_{K-1}) \tau_K \le T. \label{equ:gap-indep-cd}
	\end{align}
	Plugging \eqref{equ:gap-indep-CS} and \eqref{equ:gap-indep-cd} into \eqref{equ:inden-L-bound}, we have 
	\begin{align}\label{equ:gap-indep-CS-Lti}
		\eqref{equ:inden-L-bound} 
		\le \sqrt{T} \sqrt{\frac{1}{d_1^2}+\ldots+\frac{1}{d_K^2}}=\sqrt{T} \sqrt{\sum_{i=1}^{K} \dfrac{1}{(K-i+1)^{2b}-(K-i)^{2b}}}.
	\end{align}
	Next, we have
	\begin{align}
		\sum_{i=1}^{K} \dfrac{1}{(K-i+1)^{2b}-(K-i)^{2b}}
		&=\sum_{x=1}^{K} \dfrac{1}{x^{2b}-(x-1)^{2b}} \notag\\
		& \overset{(a)}{\le} \int_{1}^K \dfrac{ 1}{x^{2b}-(x-1)^{2b}} dx +1 \notag\\
		& \overset{(b)}{\le} \int_{1}^K x^{1-2b} dx+1 \notag\\
		& = \frac{1}{2-2b} \left(K^{2-2b}-1\right)+1 \notag\\
		& \overset{(c)}{\le} \frac{1}{2-2b} K^{2-2b}, \label{equ:gap-indep-d-up}
	\end{align}
	where $(a)$ follows by $\dfrac{1}{x^{2b}-(x-1)^{2b}}$ decreasing in $x$, $(b)$ follows by $\dfrac{1}{x^{2b}-(x-1)^{2b}} \le x^{1-2b}$ for $x \ge 1$, $(c)$ follows by $1-\frac{1}{2-2b} \le 0$ for $1/2 \le b \le 1$.
	
	Combining \eqref{equ:gap-indep-CS-Lti} and \eqref{equ:gap-indep-d-up}, we have
	\begin{equation*}
		\eqref{equ:inden-L-bound} =\sum_{i=1}^K \sqrt{\sum_{s=1}^{\tau_i} \sA^{2b}} \le \sqrt{\frac{1}{2-2b}} K^{1-b} \sqrt{T}.
	\end{equation*}
	Thus, we complete the proof of Lemma \ref{lem:gap-indep-L-bound}.
\end{proof}

\section{Proofs in Section \ref{sec:gap-dep-up}}
\begin{proof}[of Theorem \ref{thm:gap-dep-low}]
	The proof basicly follows from Theorem 16.2 in \citet{lattimore2020bandit}. 
	For a bandit instance $v \in \mathcal{P}$, $\mu_{i^*}$ is the mean of the optimal arm. 
	Fix a suboptimal arm $i$ and let $v' \in \mathcal{P}$ with $\mu_j'=\mu_j$ for $j \neq i$ and $\mu_i'=\mu_{i^*}+\epsilon$ where $\epsilon>0$ can be arbitrarily small. Thus, the KL-divengence is  $KL(P_i,P_i')=(\Delta_i+\epsilon)^2/2$.	
	By Lemma \ref{lem:diver-decom}, we have 
	\begin{equation}\label{equ:KL-Asi-B}
		KL\left(\PR_v,\PR_{v'}\right)=\sum_{j=1}^{K} \EX_v\left[\sum_{t=1}^T A_{tj}^{2-2b}\right] KL(P_j,P_j')=\EX_v\left[\sum_{t=1}^T A_{ti}^{2-2b}\right] (\Delta_i+\epsilon)^2/2,
	\end{equation}
	where the last equation follows from the definition of $v'$.
	Let $S_{Ti} \coloneqq \sum_{t=1}^T A_{ti}, B_{Ti} \coloneqq \sum_{t=1}^T A_{ti}^{2-2b}$. For $b \le 1/2$, we have $A_{ti}^{2-2b} \le A_{ti}$. So $B_{Ti} \le S_{Ti}$. By \eqref{equ:KL-Asi-B}, we have
	\begin{equation}\label{equ:LB-dep-key}
		KL\left(\PR_v,\PR_{v'}\right)= \EX_v\left[B_{Ti}\right] (\Delta_i+\epsilon)^2/2 \le \EX_v\left[S_{Ti}\right] (\Delta_i+\epsilon)^2/2.
	\end{equation}
	By Theorem 14.2 in \citet{lattimore2020bandit}, we have for any event $A$,
	\begin{equation}\label{equ:LB-dep-LeCam}
		\PR_v(A)+\PR_{v'}(A^c) \ge \frac{1}{2} \exp\left(-KL\left(\PR_v,\PR_{v'}\right)\right).
	\end{equation}
	Now we choose $A=\left\{S_{Ti} > T/2\right\}$. Under event $v$, arm $i$ is suboptimal, the regret has the lower bound 
	\begin{equation}\label{equ:LB-dep-v}
		R_{\pi}(T,v) \ge \frac{T}{2} \PR_v(A) \Delta_i.
	\end{equation} 
	While under event $v'$, arm $i$ is optimal, the regret is lower bounded by
	\begin{equation}\label{equ:LB-dep-v'}
		R_{\pi}(T,v') \ge \frac{T}{2} \PR_{v'}(A^c) (\mu_i'-\mu_{i^*})=\frac{T}{2} \PR_{v'}(A^c) \epsilon.
	\end{equation} 
	Combining \eqref{equ:LB-dep-v} and \eqref{equ:LB-dep-v'}, we have
	\begin{align}
		R_{\pi}(T,v)+R_{\pi}(T,v') 
		&\ge \frac{T}{2} \left(\PR_v(A) \Delta_i+\PR_{v'}(A^c) \epsilon\right) \notag\\
		&\ge \frac{T}{4} \min\{\Delta_i, \epsilon\} \left(\PR_v(A)+\PR_{v'}(A^c) \right) \notag\\
		&\overset{(a)}{\ge}\frac{T}{4} \min\{\Delta_i, \epsilon\} \exp\left(-KL\left(\PR_v,\PR_{v'}\right)\right) \notag\\
		&\overset{(b)}{\ge}\frac{T}{4} \min\{\Delta_i, \epsilon\} \exp\left(-\EX_v\left[S_{Ti}\right] (\Delta_i+\epsilon)^2/2\right) \label{equ:LB-dep-Reg-add},
	\end{align}
	where $(a)$ follows from \eqref{equ:LB-dep-LeCam}, $(b)$ follows from \eqref{equ:LB-dep-key}.
	Rearranging \eqref{equ:LB-dep-Reg-add} and taking the limit inferior leads to 
	\begin{align}
		\liminf_{T \rightarrow \infty} \frac{E_v[S_{Ti}]}{\log T} 
		&\ge \frac{2}{(\Delta_i+\epsilon)^2} \liminf_{T \rightarrow \infty} \log \left(\dfrac{T\min\{\Delta_i,\epsilon\}}{4(R_{\pi}(T,v)+R_{\pi}(T,v'))}\right) \big/ \log T \notag\\
		& = \frac{2}{(\Delta_i+\epsilon)^2} \liminf_{T \rightarrow \infty} \left(1+\frac{\log(\min\{\Delta_i,\epsilon\}/4)}{\log T}-\frac{\log(R_{\pi}(T,v)+R_{\pi}(T,v'))}{\log T}\right) \notag\\
		& = \frac{2}{(\Delta_i+\epsilon)^2}  \left(1-\limsup_{T \rightarrow \infty}\frac{\log(R_{\pi}(T,v)+R_{\pi}(T,v'))}{\log T}\right) \notag\\
		& = \frac{2}{(\Delta_i+\epsilon)^2}, \label{equ:LB-dep-Reg-inf}
	\end{align}
	where the last equality follows from the definition of consistency, which says that for any $p>0$ there exists a constant $C_p$ such that for a sufficiently large $T$, $R_{\pi}(T,v)+R_{\pi}(T,v') \le C_p T^p$, which implies that
	\begin{equation*}
		\limsup_{T \rightarrow \infty}\frac{\log(R_{\pi}(T,v)+R_{\pi}(T,v'))}{\log T} \le \limsup_{T \rightarrow \infty}\frac{p \log T +\log (C_p)}{\log T}=p,
	\end{equation*}
	which gives the result since $p>0$ can be arbitrarily small. Recalling the definition of regret, we have
	\begin{equation*}
		\liminf_{T \rightarrow \infty} \frac{R_{\pi}(T,v)}{\log T} = \liminf_{T \rightarrow \infty} \sum_{i:\Delta_i >0}\frac{\Delta_i E_v[S_{Ti}]}{\log T} \overset{(a)}{\ge} \sum_{i:\Delta_i >0} \frac{2 \Delta_i}{(\Delta_i+\epsilon)^2}
		\overset{(b)}{=} \sum_{i:\Delta_i >0} \frac{2}{\Delta_i},
	\end{equation*}
	where $(a)$ holds by \eqref{equ:LB-dep-Reg-inf}, $(b)$ holds by choosing arbitrarily small $\epsilon>0$.	
\end{proof}

\begin{proof}[of Theorem \ref{thm:gap-dep-up-greedy}]
	Recall the definition of regret 
	\begin{align}\label{equ:greedy-reg-decom}
		R_{\pi}(T,v)&=\sum_{t=1}^T \EX\left[\muopt-\sumiK \tA \mu_i\right]
		=T \muopt-\sumiK \EX\left[\sum_{t=1}^T \tA \mu_i\right]\notag\\
		&=T \muopt-\sumiK \EX[S_{Ti} \mu_i]
		=\sumiK \EX[S_{Ti}] \Delta_i.
	\end{align}
	By Algorithm \ref{alg:eps-greedy}, the resource allocated to arm $i$ is either $\tA$ or $1-K\tA$. So we have
	\begin{align}
		\EX[S_{Ti}]&=\frac{1}{K}+\sum_{t=2}^T \tA+ (1-K \tA) \PR\left(\mutwo \ge \hat{\mu}^{(2)}_{ti^*}\right)\notag\\
		& \le \frac{1}{K} \left(1+\sum_{t=1}^T t^{-\alpha}\right)+ \sum_{t=2}^T \PR\left(\mutwo \ge \hat{\mu}^{(2)}_{ti^*}\right)\notag\\
		& \le \frac{1}{K}\left(2+\int_{t=1}^T t^{-\alpha} dt\right) +\sum_{t=2}^T \PR\left(\mutwo \ge \hat{\mu}^{(2)}_{ti^*}\right)\notag\\
		& \le \frac{2\alpha-1}{K(\alpha-1)}+\sum_{t=2}^T \PR\left(\mutwo \ge \hat{\mu}^{(2)}_{ti^*}\right). \label{equ:greedy-STi}
	\end{align}
	%Next, we show an upper bound for $\PR\left(\mutwo \ge \hat{\mu}^{(1)}_{ti^*}\right)$ for any $t \in [T]$.
	At period $t$, we define the good events for arm $i$ that $G_{ti}=\{\mutwo \in (\mu_i-\Delta_i/2,\mu_i+\Delta_i/2)\}$.
	Under the event $G_{ti} \cap G_{ti^*}$, we have
	\begin{align*}
		\mutwo < \mu_i +\Delta/2 =  \mu_{i^*}-\Delta/2 <\hat{\mu}^{(2)}_{ti^*}.
	\end{align*}
	Thus, we have
	\begin{equation}\label{equ:greedy-mu-bad}
		\PR\left(\mutwo \ge \hat{\mu}^{(2)}_{ti^*}\right) \overset{(a)}{\le} \PR\left(\left(\mutwo \ge \hat{\mu}^{(2)}_{ti^*}\right) \cap( G_{ti} \cap G_{ti^*})\right) + \PR\left(G_{ti}^c \cup G_{ti^*}^c \right) \overset{(b)}{\le} \PR\left(G_{ti}^c \right)+\PR\left(G_{ti^*}^c\right),
	\end{equation}
	where $(a)$ holds by probability decomposition and $(b)$ holds by \\
	$\PR\left(\left(\mutwo \ge \hat{\mu}^{(2)}_{ti^*}\right) \cap\left( G_{ti} \cap G_{ti^*}\right)\right)=0$.
	The remaining is to show an upper bound for $\PR\left(G_{ti}^c \right)$ for any $i \in [K]$.
	We have 
	\begin{align}\label{equ:greedy-Gtic}
		\PR\left(G_{ti}^c \right) 
		\le \PR\left(\mutwo  \le \mu_i-\Delta_i/2\right)+\PR\left(\mutwo  \ge \mu_i+\Delta_i/2\right) = 2 \PR\left(\mutwo-\mu_i  \ge \Delta_i/2\right).
	\end{align}
	Recalling the definition of $\mutwo$ in \eqref{equ:mean-two} and the choice of $\tA$ in Algorithm \ref{alg:eps-greedy}, we have
	\begin{equation}\label{equ:gap-dep-mean}
		\mutwo-\mu_i=\frac{1}{\sum_{s=1}^t \Ind (\sA>0)} \sum_{s=1}^t \sA^{b-1} \xi_{si} \Ind(\sA>0)=\frac{1}{t} \sumst \sA^{b-1} \xi_{si}.
	\end{equation}
	Let $\epsilon_{ti}=K^{-1} t^{-\alpha}$. Since $\sumiK\epsilon_{si} \le 1$ for any $s \in [t]$, we have $\sA \ge \epsilon_{si}$ and $\sA^{b-1} \le K^{1-b} s^{\alpha(1-b)}$ for any $i \in [K]$. Since $\xi_{si}$ is $\sigma$-sub-Gaussian, by \eqref{equ:gap-dep-mean}, we have $\mutwo-\mu_i$ is $\frac{\sigma}{t} K^{1-b} \sqrt{\sum_{s=1}^t s^{2\alpha(1-b)}}$-sub-Gaussian. 
	And 
	\begin{align*}
		\sum_{s=1}^t s^{2\alpha(1-b)} 
		&\le \int_{s=1}^t s^{2\alpha(1-b)} ds+t^{2\alpha(1-b)}
		=\dfrac{t^{1+2\alpha(1-b)}-1}{1+2\alpha(1-b)}+t^{2\alpha(1-b)}\\
		& \overset{(a)}{\leq} \left(\frac{t}{1+2\alpha(1-b)}+1\right) t^{2\alpha(1-b)}
		\overset{(b)}{\leq} (t+1) t^{2\alpha(1-b)}\\
		& \leq 2 t^{1+2\alpha(1-b)}
		\le 2 t^{3-2b+\epsilon}, 
	\end{align*}
	where $(a)$ follows from $-1<0$ and $(b)$ follows from $1+\alpha(2-2b) \ge 1$. So $\mutwo-\mu_i$ is $\sigma K^{1-b} \sqrt{2t^{1-2b+\epsilon}}$-sub-Gaussian and 
	\begin{align}\label{equ:greedy-con-mu}
		\PR\left(\mutwo-\mu_i  \ge \Delta_i/2\right) \le \exp\left(-\dfrac{\Delta_i^2 t^{2b-1-\epsilon}}{16 \sigma^2 K^{2-2b}}\right) =\exp\left(-a t^{\gamma}\right),
	\end{align}
	where $a \coloneqq \Delta_i^2/(16 \sigma^2 K^{2-2b})>0$ and $\gamma \coloneqq 2b-1-\epsilon \in (0,1)$. 
	
	Let $m=\lceil 1/\gamma\rceil$, we have 
	\begin{align}
		\sum_{t=2}^T \exp\left(-a t^{\gamma}\right)
		& \overset{(a)}{\le} \int_{t=1}^{\infty} \exp\left(-a t^{\gamma}\right) dt 
		\overset{(b)}{=} \int_{t=1}^{\infty} \exp\left(-a t\right) dt^{1/\gamma} \notag\\
		& \overset{(c)}{\le} \int_{t=1}^{\infty} \exp\left(-a t\right) dt^{m}
		=m \int_{t=1}^{\infty} \exp\left(-a t\right) t^{m-1} dt \notag\\
		& =-\frac{m}{a} \int_{t=1}^{\infty} t^{m-1} d \exp\left(-a t\right) \notag\\
		&\overset{(d)}{=}-\frac{m}{a} \left(0-e^{-a}-\int_{t=1}^{\infty} \exp(-at) d t^{m-1}\right) \notag\\
		& =\frac{m}{a} \left(e^{-a}+\int_{t=1}^{\infty} \exp(-at) d t^{m-1}\right), \label{equ:greedy-sum-exp-half}
	\end{align}
	where $(a)$ holds by the function $\exp(-a t^{\gamma})$ decreasing in $t$, $(b)$ holds by changing of variable $t \rightarrow t^{1/\gamma}$, $(c)$ follows from $dt^{1/\gamma} =(1/\gamma) t^{1/\gamma-1} dt \le m t^{m-1} dt=d t^m$, $(d)$ follows from integration by parts. Thus, we have reduced $\int_{t=1}^{\infty} \exp\left(-a t\right) dt^{m}$ to $\int_{t=1}^{\infty} \exp\left(-a t\right) dt^{m-1}$. So on and so forth, we have 
	\begin{align}
		\eqref{equ:greedy-sum-exp-half}&=\frac{m}{a} \left(e^{-a}+\frac{m-1}{a}\left(e^{-a}+\int_{t=1}^{\infty} \exp(-at) d t^{m-2}\right)\right) \notag\\
		& \overset{(f)}{=} \left(\frac{m}{a}+\frac{m(m-1)}{a^2}+\ldots+\frac{m!}{a^m}\right) e^{-a} \notag\\
		& \le \left(\frac{m}{a}+\frac{m^2}{a^2}+\ldots+\frac{m^m}{a^m}\right) e^{-a} \notag\\
		& \overset{(g)}{=}\frac{m}{a} \left(\frac{(m/a)^m-1}{m/a-1}\right) e^{-a} \notag\\
		& \overset{(h)}{\le}\frac{m}{a} \left(\frac{(m/a)^m-1}{m/a-1}\right) \notag\\
		& =\frac{m}{m-a} \left((m/a)^m-1\right), \label{equ:greedy-sum-exp}
	\end{align}
	where $(f)$ follows from $\int_{t=1}^{\infty} \exp(-at) d t=e^{-a}$, $(g)$ follows from the summation for geometric sequence, $(h)$ follows from $e^{-a} < 1$. 
	If $a \le m/2$, then $m/(m-a) \le 2$ and $\eqref{equ:greedy-sum-exp} \le 2(m/a)^m$. If $a \ge m/2$, then $\eqref{equ:greedy-sum-exp} \le 2+2^2+\ldots+2^m \le 2^{m+1}$. Thus, we have  
	\begin{equation}\label{equ:greedy-ma}
		\eqref{equ:greedy-sum-exp} \le 2(m/a)^m+2^{m+1}.
	\end{equation}
	Combining all the inequalities, we have
	\begin{align*}
		& R_{\pi}(T,v)
		\overset{(a)}{=}\sumiK \EX[S_{Ti}] \Delta_i\\
		& \overset{(b)}{\le} \frac{2\alpha-1}{K(\alpha-1)} \sumiK \Delta_i+\sumiK\sum_{t=2}^T \PR\left(\mutwo \ge \hat{\mu}^{(1)}_{ti^*}\right)\Delta_i\\
		& \overset{(c)}{\le} \frac{2\alpha-1}{K(\alpha-1)} \sumiK \Delta_i+4 \sumiK \sum_{t=2}^T\PR\left(\mutwo-\mu_i  \ge \Delta_i/2\right) \Delta_i\\
		& \overset{(d)}{\le} \frac{2\alpha-1}{K(\alpha-1)} \sumiK \Delta_i+4 \sumiK \sum_{t=2}^T\exp(-at^{\gamma}) \Delta_i\\
		& \overset{(e)}{\le} \frac{2\alpha-1}{K(\alpha-1)} \sumiK \Delta_i+4 \sumiK \frac{m}{m-a} \left((m/a)^m-1\right) \Delta_i\\
		& \overset{(f)}{\le}
		\frac{2(1-b+\epsilon)}{K \epsilon}\sumiK \Delta_i+8 \sumiK \left((m/a)^m+2^{m}\right) \Delta_i\\
		&\! \overset{(g)}{=}\!
		\left(\frac{2(1\!- \!b\!+\! \epsilon)}{K \epsilon}\!+\!2^{\lceil 1/(2b-1-\epsilon)\rceil+4}\right) \sumiK \Delta_i\!+\!8 \sumiK \left(\frac{16 \sigma^2 K^{2-2b} \lceil 1/(2b \!-\!1\!-\!\epsilon) \rceil}{\Delta_i^2}\right)^{\lceil 1/(2b\!-\!1\!-\!\epsilon)\rceil} \!\Delta_i,
	\end{align*}
	where $(a)$ holds by \eqref{equ:greedy-reg-decom}, $(b)$ holds by \eqref{equ:greedy-STi}, $(c)$ holds by \eqref{equ:greedy-mu-bad},\eqref{equ:greedy-Gtic}, $(d)$ holds by \eqref{equ:greedy-con-mu}, $(e)$ holds by \eqref{equ:greedy-sum-exp}, $(f)$ hols by \eqref{equ:greedy-ma} and $\alpha=1+\epsilon/(2-2b)$, $(g)$ holds by $m=\lceil 1/(2b-1-\epsilon)\rceil$. Thus, the proof of Theorem \ref{thm:gap-dep-up-greedy} is completed.
\end{proof}

\end{document}